\documentclass[letterpaper]{article} % DO NOT CHANGE THIS
\usepackage{aaai2026}  % DO NOT CHANGE THIS
\usepackage{times}  % DO NOT CHANGE THIS
\usepackage{helvet}  % DO NOT CHANGE THIS
\usepackage{courier}  % DO NOT CHANGE THIS
\usepackage[hyphens]{url}  % DO NOT CHANGE THIS
\usepackage{graphicx} % DO NOT CHANGE THIS
\usepackage{natbib}  % DO NOT CHANGE THIS AND DO NOT ADD ANY OPTIONS TO IT
\usepackage{caption} % DO NOT CHANGE THIS AND DO NOT ADD ANY OPTIONS TO IT
\usepackage{algorithm}
\usepackage{algorithmic}

\usepackage{newfloat}
\usepackage{listings}
\DeclareCaptionStyle{ruled}{labelfont=normalfont,labelsep=colon,strut=off} % DO NOT CHANGE THIS
\floatstyle{ruled}
\newfloat{listing}{tb}{lst}{}
\floatname{listing}{Listing}
\usepackage[utf8]{inputenc} % allow utf-8 input
\usepackage{booktabs}       % professional-quality tables
\usepackage{amsfonts}       % blackboard math symbols
\usepackage{amsmath,amsthm,mathtools}	% Advanced maths commands
\usepackage{amssymb,dsfont}	% Extra maths symbols
\usepackage{nicefrac}       % compact symbols for 1/2, etc.
\usepackage{microtype}      % microtypography
\usepackage{xcolor}         % colors
\usepackage{tikz-cd}        % for drawing a diagram
\usepackage{eqparbox}       % for algorithms
\usepackage{lipsum} % for dummy text
\usepackage{bm}

\usepackage{hyperref}

\usepackage{srcltx}
\usepackage{etoolbox}
\usepackage{comment}
\usepackage{mathrsfs}
\usepackage{enumerate}
\usepackage[shortlabels]{enumitem}
\usepackage{xargs}
\usepackage{multirow}
\usepackage{subcaption}
\usepackage{upgreek}
\usepackage{cleveref}

\definecolor{codegreen}{rgb}{0,0.6,0}
\definecolor{codegray}{rgb}{0.5,0.5,0.5}
\definecolor{codepurple}{rgb}{0.58,0,0.82}
\definecolor{backcolour}{rgb}{0.95,0.95,0.92}

\lstdefinestyle{mystyle}{
    backgroundcolor=\color{backcolour},   
    commentstyle=\color{codegreen},
    keywordstyle=\color{magenta},
    numberstyle=\tiny\color{codegray},
    stringstyle=\color{codepurple},
    basicstyle=\ttfamily\footnotesize,
    breakatwhitespace=false,         
    breaklines=true,                 
    captionpos=b,                    
    keepspaces=true,                 
    numbers=left,                    
    numbersep=5pt,                  
    showspaces=false,                
    showstringspaces=false,
    showtabs=false,                  
    tabsize=2
}

\def\rset{\mathbb{R}}
\def\diag{\operatorname{diag}}

\def\P{\mathbb{P}}
\newcommand{\eqsp}{\;}
\def\trace{\operatorname{tr}}
\def\rank{\operatorname{rk}}

\usepackage[disable]{todonotes}

\DeclareMathOperator*{\argmax}{arg\,max}

\crefname{section}{Section}{Sections}
\Crefname{section}{Section}{Sections}

\newtheorem{theorem}{Theorem}
\crefname{theorem}{theorem}{Theorems}
\Crefname{theorem}{Theorem}{Theorems}

\newtheorem{definition}{Definition}
\crefname{definition}{definition}{definitions}
\Crefname{Definition}{Definition}{Definitions}

\newtheorem{lemma}[theorem]{Lemma}
\crefname{lemma}{lemma}{lemmas}
\Crefname{lemma}{Lemma}{Lemmas}

\newtheorem{remark}[theorem]{Remark}
\crefname{remark}{remark}{remarks}
\Crefname{remark}{Remark}{Remarks}

\crefname{corollary}{corollary}{corollaries}
\Crefname{corollary}{Corollary}{Corollaries}

\crefname{proposition}{proposition}{propositions}
\Crefname{proposition}{Proposition}{Propositions}

\newtheorem{example}{Example}
\crefname{example}{example}{examples}
\Crefname{Example}{Example}{Examples}

\crefname{figure}{figure}{figures}
\Crefname{Figure}{Figure}{Figures}

\crefname{table}{table}{tables}
\Crefname{Table}{Table}{Tables}

\crefname{equation}{equation}{equations}
\Crefname{Equation}{Equation}{Equations}

\newenvironment{explanation}[1][Explanation]{%
  \begin{trivlist}
  \item[\hskip\labelsep\itshape #1.]\normalfont
}{%
  \hfill\qedsymbol\end{trivlist}
}

\crefname{assum}{A\hspace{-2pt}}{A\hspace{-2pt}}
\crefname{assumb}{B\hspace{-2pt}}{B\hspace{-2pt}}
\crefname{assumUGE}{UGE\hspace{-1pt}}{UGE\hspace{-1pt}}
\crefname{assumID}{IND\hspace{-1pt}}{IND\hspace{-1pt}}
\crefname{assumUE}{UE\hspace{-1pt}}{UE\hspace{-1pt}}
\crefname{assumSUP}{M\hspace{-1pt}}{M\hspace{-1pt}}

\newlist{renumerate}{enumerate}{3}
\setlist[renumerate]{wide, labelwidth=!, labelindent=0pt,label=(\roman*)}

\newlist{aenumerate}{enumerate}{3}
\setlist[aenumerate]{wide, labelwidth=!, labelindent=0pt,label=(\arabic*)}

\newlist{aaenumerate}{enumerate}{3}
\setlist[aaenumerate]{wide, labelwidth=!, labelindent=0pt,label=(\alph*)}

\newlist{aenumerateSpace}{enumerate}{3}
\setlist[aenumerateSpace]{wide, labelwidth=!,label=(\arabic*)}

\newlist{benumerate}{enumerate}{3}
\setlist[benumerate]{wide, labelwidth=!, labelindent=0pt,label=$\bullet$}

\title{Matrix-Free Two-to-Infinity and One-to-Two Norms Estimation}
\author {
    Askar Tsyganov\textsuperscript{\rm 1},
    Evgeny Frolov\textsuperscript{\rm 2,\rm 1},
    Sergey Samsonov\equalcontrib\textsuperscript{\rm 1},
    Maxim Rakhuba\equalcontrib\textsuperscript{\rm 1}
}
\affiliations {
    \textsuperscript{\rm 1}HSE University, Moscow, Russia\\
    \textsuperscript{\rm 2}Personalization Technologies, Moscow, Russia\\
    \{atsyganov, e.frolov, svsamsonov, mrakhuba\}@hse.ru
}

\begin{document}

\maketitle

\begin{abstract}
In this paper, we propose new randomized algorithms for estimating the two-to-infinity and one-to-two norms in a matrix-free setting, using only matrix-vector multiplications. Our methods are based on appropriate modifications of Hutchinson's diagonal estimator and its Hutch++ version. We provide oracle complexity bounds for both modifications. We further illustrate the practical utility of our algorithms for Jacobian-based regularization in deep neural network training on image classification tasks. We also demonstrate that our methodology can be applied to mitigate the effect of adversarial attacks in the domain of recommender systems.
\end{abstract}

\begin{links}
    \link{Code}{https://github.com/fallnlove/TwoToInfinity}
\end{links}

\section{Introduction}
\label{sec:intro}

In recent years, there has been growing interest in randomized linear algebra techniques \cite{halko2011finding,martinsson2020randomized} for estimating matrix functions without explicit access to the matrix entries, see e.g. \cite{hutchinson1989stochastic,bekas2007estimator}. This setting, known as \emph{matrix-free}, assumes access to an oracle that computes matrix-vector products with a matrix $A$, and in some cases also with its transpose $A^\top$. The goal is to approximate important characteristics of $A$, for example, norms, trace or spectral density, using only these products. Such a framework is essential in modern machine learning, where matrices such as the Jacobian of deep neural networks are prohibitively large to form explicitly but allow efficient computation of matrix-vector products via automatic differentiation (autograd), see e.g. \cite{baydin2018automatic}.
\par
In this paper, we focus on matrix-free estimation of two operator norms: $\|\cdot\|_{2\to\infty}$ and $\|\cdot\|_{1\to2}$. Formally, for a matrix $A \in \mathbb{R}^{d\times n}$, its two-to-infinity norm can be defined as
\begin{equation*}
\|A\|_{2\to\infty} = \sup_{x \neq 0} \frac{\|Ax\|_{\infty}}{\|x\|_{2}} = \max_{i\in[d]} \|A_{i}\|_2\eqsp,
\end{equation*}
where $A_i$ denotes the $i$-th row of $A$. Given the identity
\begin{equation*}
\|A\|_{2\to\infty} = \|A^\top\|_{1\to2}\eqsp,
\end{equation*}
it suffices to concentrate on the estimation of the two-to-infinity norm. Compared to classical matrix norms such as the spectral or Frobenius norm, the two-to-infinity norm provides finer control over the row-wise structure of a matrix. Indeed, when dealing with \emph{tall} matrices, that is, matrices $A \in \rset^{d \times n}$ with $d \gg n$, it is natural that $\|A\|_{2\to\infty}$ does not scale with $d$, contrary to spectral and Frobenius norm of this matrix. In such cases, bounding the two-to-infinity norm ensures that the norm of each row is tightly controlled. This localized control is especially useful when studying theoretical properties of various algorithms \cite{cape2019two,pensky2024davis}, see details in the related work section. In this paper, we focus on methodological and statistical aspects of estimating the two-to-infinity norm, rather than on its theoretical utility. Our main contributions are as follows:

\begin{itemize}
    \item We introduce a novel randomized algorithm tailored specifically for the $\|\cdot\|_{1 \to 2}$ and $\|\cdot\|_{2 \to \infty}$ norms estimation under the matrix-free setting. Our method enjoys provable convergence guarantees and empirically demonstrates reliable performance. 
    \item We apply suggested randomized estimators as regularizers in image classification and recommender systems problems. For image classification with deep neural networks, our method achieves better generalization performance compared to prior Jacobian regularization techniques \cite{hoffman2019robust,roth2020adversarial}. In the recommender systems domain, we consider UltraGCN-type architectures \cite{mao2021ultragcn} and show that employing our regularizer increases the robustness of the algorithm to adversarial attacks, following the pipeline of \cite{he2018adversarial}.
\end{itemize}

The rest of the paper is structured as follows. We discuss related work in \Cref{sec:related-work}. Then, in \Cref{sec:main_algo}, we present our main algorithm, TwINEst (see \Cref{algo:twinest}), and analyze its sample complexity. In \Cref{sec:improved_algo}, we provide a variance‐reduced version of the TwINEst algorithm and study the theoretical properties of the modified algorithm. Finally, we present numerical results in \Cref{sec:experiments}. Proofs of the theoretical results and additional numerical experiments are provided in the supplementary material.

\paragraph{Notations.} For a vector $x \in \rset^d$, $\|x\|_p = (\sum_{i=1}^d |x_i|^p)^{1/p}$ denotes the $\ell_p$-norm ($p\geq1$), $\|x\|_{\infty} = \max_{i} |x_i|$ denotes the $\ell_\infty$-norm.
For a matrix $A \in \rset^{d \times n}$, $A_i$ denotes the $i$-th row of $A$, and $\|A\|_F = (\sum_{i=1}^d\sum_{j=1}^n A_{ij}^2)^{1/2}$ denotes its Frobenius norm. We define the induced norms as
$\|A\|_{p \to q} := \sup_{x \neq 0} \|Ax\|_q/\|x\|_p$.
For example, $\|A\|_{2 \to \infty}$ is equal to the maximum $\ell_2$ norm of the rows and $\|A\|_{1 \to 2}$ is equal to the maximum $\ell_2$ norm of the columns.
We define the max-norm as $\| A \|_{\text{max}} := \min_{U, V :\, A = UV^\top} \| U \|_{2 \to \infty} \| V \|_{2 \to \infty}$.
For matrices $A, B \in \mathbb{R}^{d \times n}$, $A \odot B$ denotes the Hadamard product (element-wise). For $d \in \mathbb{N}$, we denote $[d] = \{1,\ldots,d\}$. A random variable $\xi$ is called a Rademacher random variable if $\P(\xi = 1) = \P(\xi = -1) = \frac{1}{2}$. The vector $e_i := (0,\dots,0,1,0,\dots,0)^{\top}$, with $1$ in position $i$ and $0$ elsewhere, denotes the $i$‑th standard basis vector in $\mathbb{R}^n$.

\section{Related Work}
\label{sec:related-work}
A classical problem in the matrix-free setting is the stochastic trace estimation \cite{hutchinson1989stochastic}, in which the trace of a matrix is approximated using a few matrix-vector products with random vectors. Several improvements on this approach have been developed, including variance-reduced methods such as Hutch++ \cite{meyer2021hutch++}, which exploit low-rank structure to accelerate convergence, and dynamic algorithms \cite{dharangutte2021dynamic,woodruff2022optimal}, which adaptively allocate samples to achieve higher accuracy. Related work also addressed the problem of estimating the spectral norm of a matrix using structured estimators based on rank-one vectors, see \cite{bujanovic2021norm}. Another closely related problem is the estimation of the matrix diagonal \cite{bekas2007estimator,baston2022stochastic,dharangutte2023tight}, for which recent works have provided algorithmic improvements and theoretical guarantees. 
\par 
The two-to-infinity norm is widely used as a theoretical tool to study statistical guarantees for algorithms in various areas of high-dimensional statistics and learning theory. Particular applications include bandit problems with specific reward structures \cite{jedra2024low}, singular subspace recovery \cite{cape2019two}, and clustering \cite{pensky2024davis}. 
\par
In machine learning algorithms, the two-to-infinity norm most often appears as part of the max-norm (see the notation section), an important tool for algorithm regularization. Its applications include, in particular, the matrix completion problem via matrix factorization \cite{srebro2004maximum}. In this setting, one can show that the max-norm serves as a surrogate for the rank of the approximator matrix. This idea was further developed in \cite{lee2010practical}, where the authors introduce gradient-based algorithms, such as the projected gradient method and the proximal point method, for solving the regularized matrix factorization problem. Max-norm regularizers are also an important part of many online frameworks for such problems, often combined with additional sparsity-inducing constraints, as in \cite{shen2014online}.
\par
Estimating matrix operator norms $\|A\|_{p \to q}$ induced by different combinations of $p$ and $q$ has recently attracted a lot of attention \cite{higham2000block,bujanovic2021norm,bresch2024matrix,naumov2025upper}. A foundational contribution in this direction is the work \cite{boyd1974power}, which proposes a general iterative algorithm for estimating $\|\cdot\|_{p \to q}$ norm for arbitrary $p$ and $q$, except for the cases of $\|\cdot\|_{1 \to 2}$ and $\|\cdot\|_{2 \to \infty}$ norms. Their approach is based on a generalization of the classical power method. Papers \cite{higham1992estimating,roth2020adversarial} extend the methodology of \cite{boyd1974power} and apply it to more specialized norms, including the $\|\cdot\|_{1 \to 2}$ and $\|\cdot\|_{2 \to \infty}$ norm cases. While \cite{boyd1974power} proves theoretical convergence results for certain types of matrices, the authors of \cite{higham1992estimating,roth2020adversarial} do not provide any theoretical guarantees for the $\|\cdot\|_{1 \to 2}$ and $\|\cdot\|_{2 \to \infty}$ cases. Furthermore, we provide an example below demonstrating that, when estimating the $\|\cdot\|_{2 \to \infty}$ norm, this method diverges with positive probability even when applied to a simple diagonal matrix.

\section{Two-to-Infinity Norm Estimation}
\label{sec:main_algo}

Here, we briefly introduce the adaptive power method proposed by \cite{higham1992estimating,roth2020adversarial}. Full pseudocode is provided in the supplementary paper. The method begins with an initial random vector $X_0$ (typically drawn from a Gaussian distribution), followed by an iterative update process:
\begin{equation*}
Y^i = \operatorname{dual}_{\infty}(AX^{i - 1}),
\quad X^i = \operatorname{dual}_2(A^\top Y^i),
\end{equation*}
where the dual operator is defined as
\begin{equation*}
\operatorname{dual}_p(x) =
\begin{cases}
\operatorname{sign}(x) \odot |x|^{p - 1} / \|x\|_p^{p - 1}, & \text{if } p < \infty, \\
|\mathcal{I}|^{-1} \operatorname{sign}(x) \odot \boldsymbol{1}_{\mathcal{I}}, & \text{if } p = \infty,
\end{cases}
\end{equation*}
where $\mathcal{I} := \{ i \in [d] : |x_i| = \|x\|_\infty \}$, and $\boldsymbol{1}_{\mathcal{I}} := \sum_{i \in \mathcal{I}} e_i$ is an indicator vector over the set $\mathcal{I}$. The final estimate after $m$ iterations is obtained by computing $\|AX^m\|_{\infty}$.
\begin{example}[Divergence of the Adaptive Power Method]
\label{example:divergence_power_method}
Let $A \in \mathbb{R}^{2 \times 2}$ be given by
\begin{equation*}
A = \begin{pmatrix}
2 & 0 \\
0 & 1 \\
\end{pmatrix}.
\end{equation*}
Then, with probability at least $0.295$, the Adaptive Power Method diverges when applied to the matrix $A$.
\end{example}
\begin{explanation}
Let $X = AX^0$, where $X^0 \sim \mathcal{N}(0, \operatorname{I}_2)$ is the initial vector, and $X = (X_1, X_2)^{\top}$. Then:
\begin{multline*}
\mathbb{P}(|X_1| < |X_2|) = \mathbb{P}(2|X^0_1| < |X^0_2|) = \mathbb{P}\left(\left|\frac{X^0_1}{X^0_2}\right| < \frac{1}{2} \right)
\\=
\int_{-1/2}^{1/2} \frac{1}{\pi (1 + t^2)} \text{ d} t = \frac{2}{\pi} \arctan(1/2) \approx 0.2951,
\end{multline*}
since the random variable $X^0_1 / X^0_2$ has a Cauchy distribution.
By the definition of $\operatorname{dual}_\infty$, we then have $Y^1 = \operatorname{sign}(X) \odot (0, 1)^\top$ with probability at least $0.295$. Then it is easy to check that
\begin{equation*}
X^1 = \frac{A^\top Y^1}{\|A^\top Y^1\|_2} = \operatorname{sign}(X)\odot (0, 1)^\top.
\end{equation*}
It is easy to verify that $X^i = Y^i = \operatorname{sign}(X) \odot (0, 1)^\top$ for any $i \geq 1$. Consequently, the final estimate with probability $\geq 0.295$ equals
\begin{equation*}
L = \| A X^m \|_{\infty} = 1,
\end{equation*}
which is incorrect, as $\|A\|_{2 \to \infty} = 2$.
\end{explanation}

\subsection{Hutchinson's Diagonal Estimator}
\label{subsec:hutchinson_diagonal}

Unlike the iterative algorithms, such as the adaptive power method, we focus on a stochastic estimator of the two-to-infinity norm, which is accurate in the limit of large number of matrix-vector products. Our algorithms build upon a well-known technique for estimating the diagonal of a square matrix using only matrix-vector products. This technique is known as the Hutchinson diagonal estimator \cite{bekas2007estimator,baston2022stochastic,dharangutte2023tight}. In this section, we introduce the Hutchinson diagonal estimator and provide concentration inequalities that are useful for our subsequent analysis.

\begin{definition}
\label{def:hutch_diag_est}
Let $X^1$, $\dots$, $X^m \in \{-1, 1\}^d$ be independent Rademacher random vectors. For a square matrix $A \in \mathbb{R}^{d \times d}$, the Hutchinson's diagonal estimator $D^m(A) \in \mathbb{R}^d$ is defined as
\begin{equation*}
    D^m(A) \coloneqq \frac{1}{m} \sum_{i=1}^{m} X^i \odot (A X^i).
\end{equation*}
\end{definition}
Notably, in this definition, instead of Rademacher vectors, it is possible to use Gaussian vectors or any mean-zero random vectors with an identity covariance matrix. However, we prefer Rademacher vectors because they yield an estimate with minimal variance (see Proposition 1 in \cite{hutchinson1989stochastic}).

This estimator is a natural extension of the classical Hutchinson method for trace estimation \cite{hutchinson1989stochastic}. It provides an unbiased estimate of the diagonal, moreover, its variance can be computed, for any $i \in [d]$, as 
\begin{equation*}
\operatorname{Var}[D^1_i(A)] = \sum_{j \neq i} A_{ij}^2.
\end{equation*}
Furthermore, the estimator enjoys high-probability error bounds that quantify its deviation from the true diagonal. In particular, we rely on the following result from~\cite{dharangutte2023tight}, which provides the following bound for the $\ell_2$ norm of the Hutchinson's estimator error:

\begin{theorem}[Theorem~1 in \cite{dharangutte2023tight}]
\label{theorem:tight_diagonal_bound}
Let $A \in \mathbb{R}^{d \times d}$, $m \in \mathbb{N}$, $\delta \in (0, 1]$. Then with probability at least $1 - \delta$:
\begin{equation*}
    \|D^m(A) - \diag(A) \|_2 \leq c \sqrt{\frac{\log(2 / \delta)}{m}} \| A - \diag(A) \|_F,
\end{equation*}
where $c$ is an absolute constant.
\end{theorem}

\subsection{Our method}
Now we describe our strategy for estimating $\|\cdot\|_{2 \to \infty}$ norm. The main idea is that the diagonal entries of the matrix $AA^\top$ correspond to the
squared $\ell_2$ norms of the rows of $A$. Therefore, the $\|\cdot\|_{2\to\infty}$ norm can be equivalently expressed as
\begin{equation*}
    \|A\|_{2\to\infty}^2 = \max_{i \in [d]} \operatorname{diag}(AA^\top)_i.
\end{equation*}
This identity suggests a natural strategy: instead of computing all row norms explicitly, we can estimate the diagonal of $AA^\top$ using the Hutchinson method, which only requires matrix-vector products with $A$ and $A^\top$. The final estimate of the $\|A\|_{2\to\infty}$ norm is then obtained by taking the maximum of the estimated diagonal.

However, estimating the maximum value through the direct application of Hutchinson's method introduces high variance, leading to a noisy approximation. To mitigate this, we can eliminate one source of randomness in the final estimate. Namely, let $D$ be the estimate of the diagonal of $ AA^\top $. While the entries of $ D $ are typically noisy, we can reduce variance by avoiding direct use of $ \max_i D_i $. Instead, we first identify the (random) index of the maximum estimated value, $ j = \arg\max_{\mathrm{i}} D_\mathrm{i} $, and then compute $\| A^{\top} e_j \|_{2} = \| A_j \|_2$. This approach significantly improves the quality of the estimate, which is supported by the ablation study carried out in \Cref{exp:synthetic}. The outlined procedure is referred to as the TwINEst algorithm and is presented in \Cref{algo:twinest}. Note that the algorithm steps given in lines $2$ to $4$ can be done either in parallel, or in a matrix form, if it is possible to store the matrix $X = [X^1,\ldots,X^m] \in \rset^{d \times m}$ and compute $X \odot AA^{\top} X$. This is the main practical advantage of our method over power-iteration based algorithms \cite{roth2020adversarial}.

\begin{algorithm}[t!]
\caption{TwINEst: \textbf{Tw}o-to-\textbf{I}nfinity \textbf{N}orm \textbf{Est}imation}
\label{algo:twinest}
\begin{algorithmic}%
  \REQUIRE{
    \emph{}\\
    Matrix-vector multiplication oracle for \( A \in \mathbb{R}^{d \times n} \), \\
    Matrix-vector multiplication oracle for \( A^T \in \mathbb{R}^{n \times d} \), \\
    Positive integer \( m \in \mathbb{N} \): number of random samples.
  }
  \ENSURE{
    \emph{}\\
    An estimate of the $\|A\|_{2 \to \infty}$ norm.
  }
\end{algorithmic}%
\begin{algorithmic}[1]%
\STATE Sample $m$  random Rademacher vectors $X^1$, $\dots$, $X^m$, where each \( X^i \in \{-1, 1\}^d \)
% \STATE Initialize \( D = 0 \)
\FOR{each \( i = 1, 2, \dots, m \)}
    \STATE Compute \( t_i = X^i \odot AA^\top X^i \)
\ENDFOR
\STATE Compute $D = \frac{1}{m} \sum_{i=1}^{m} t_i \in \rset^{d}$\\
\COMMENT{\(D\) - estimate of the \(AA^\top\) diagonal}
\STATE Find \( j = \arg\max_{\mathrm{i}} D_\mathrm{i} \)
\STATE Compute $L = \| A^\top e_j \|_2$\\
\COMMENT{\(e_j\) - is the \(j\)-th standard basis vector}
\RETURN \( L \)
\end{algorithmic}%
\end{algorithm}

We now establish upper bounds on the sample complexity of TwINEst. In the context of randomized numerical linear algebra, sample complexity typically refers to the number of matrix-vector multiplications required to approximate some quantity within a specified error tolerance and failure probability. Beyond these bounds, we also derive an oracle complexity: the number of matrix–vector multiplications sufficient to return the correct answer with failure probability $\delta$, expressed in terms of the gap $\Delta$ between the largest squared $\ell_2$-norm among the rows of $A$ and the squared $\ell_2$-norm of the nearest non-maximal row.
Formally, with $M = \max_{i} \| A_i \|_2^2$, we define
\begin{equation*}
\Delta = M - \max_{i \,:\, \|A_i\|_2^2 < M} \| A_i \|_2^2.
\end{equation*}
A large $\Delta$ makes the largest $\ell_{2}$-norm easy to detect, reducing oracle complexity, whereas a small $\Delta$ requires more samples, since the top norms are close.

We now establish the oracle complexity of TwINEst, namely the number of samples required to recover the exact value of the matrix norm $\| \cdot \|_{2 \to \infty}$ with high probability.

\begin{theorem}[TwINEst Oracle Complexity]
\label{theorem:twinest_bound}
    Let $A \in \mathbb{R}^{d \times n}$, $m \in \mathbb{N}$. Let $T^m(A)$ be the result of Algorithm~\ref{algo:twinest} based on $m$ random vectors. Then, it suffices to take  
    \begin{equation*}
        m > \frac{8\log (2d / \delta)}{\Delta^2} \|AA^{\top} - \diag(AA^{\top})\|_{2 \to \infty}^2
    \end{equation*}
    to ensure $T^m(A) = \| A \|_{2 \to \infty}$ with probability at least $1 - \delta$.
\end{theorem}

\paragraph{Discussion.} The proof of \Cref{theorem:twinest_bound} is provided in the supplementary paper. Importantly, the structure of TwINEst allows us not only to bound the probability of deviating from the true value, but also to bound the probability that our algorithm returns the exact value. Notably, when $\Delta = 0$, identifying the maximum row norm becomes trivial, and the algorithm always returns the correct answer.

\begin{remark}[TwINEst Deviation Bound]
\label{remark:twinest_error_complexity}
    Let $A \in \mathbb{R}^{d \times n}$, $m \in \mathbb{N}$, and $\varepsilon > 0$. Let $T^m(A)$ be the result of Algorithm~\ref{algo:twinest} based on $m$ random vectors. Then, it suffices to take  
    \begin{equation*}
        m > \frac{8\log (2d / \delta)}{\varepsilon^2} \|AA^{\top} - \diag(AA^{\top})\|_{2 \to \infty}^2
    \end{equation*}
    to ensure $|(T^m(A))^2 - \| A \|_{2 \to \infty}^2| < \varepsilon$ with probability at least $1 - \delta$.
\end{remark}
\begin{proof}
    Proof is analogous to the proof of \Cref{theorem:twinest_bound}.
\end{proof}

Depending on $\Delta$, either \Cref{remark:twinest_error_complexity} or \Cref{theorem:twinest_bound} may be more useful. Results similar to those in \Cref{remark:twinest_error_complexity} have previously been obtained for the Hutchinson estimator \cite{roosta2015improved,jiang2021optimal}, Hutch++ \cite{meyer2021hutch++}, and several other methods. To the best of our knowledge, \Cref{theorem:twinest_bound} is the first bound on the oracle complexity of randomized estimation of the two-to-infinity norm.

\section{Improved Algorithm}
\label{sec:improved_algo}
In this section, we present an improved version of our algorithm, which incorporates the variance reduction technique of the Hutch++ method \cite{meyer2021hutch++}. Bearing similarity with the Hutch++ method, we call this modification TwINEst++. Rather than estimating the diagonal with the classical Hutchinson estimator, we adopt the unbiased Hutch++ diagonal modification \cite{han2024stochastic}. The key idea of this technique is to approximate the dominant low‑rank part of the matrix using a few random vectors, explicitly compute its diagonal, and estimate the diagonal of the residual term using Hutchinson's estimator. More precisely, let $S \in \mathbb{R}^{d \times r}$ consist of $r$ independent Rademacher vectors. Multiplying by $AA^\top$, we form the matrix $AA^{\top}S$, whose column span approximates the dominant eigenspace of $AA^{\top}$ with high probability. To approximate this eigenspace, we compute the thin $QR$-decomposition of the matrix $AA^{\top}S$ and further rely on the factor $Q \in \mathbb{R}^{d \times r}$. Using the orthogonal projector $P = QQ^{\top}$, we decompose $AA^\top$ into a low-rank approximation and a residual:
\begin{equation}
\label{eq:low_rank_decomposition}
    AA^\top = \underbrace{AA^{\top}P}_{\text{low-rank}} + \underbrace{AA^{\top}(I - P)}_{\text{residual term}}.
\end{equation}
The diagonal of the low-rank part can be computed explicitly:
\begin{equation*}
    \diag (AA^\top P) = (A (A^\top Q) \odot Q) \cdot (1 \ 1 \ ... \ 1)^\top,
\end{equation*}
this requires only $\mathcal{O}(dr)$ floating-point operations and $2r$ matrix-vector products with matrices $A$ and $A^\top$. To estimate the diagonal of the matrix $AA^{\top}$, we combine this exact computation with a stochastic estimate of the residual term:
\begin{equation*}
    \diag(AA^{\top})
    \approx
    \diag(AA^{\top}P)
    +
    D^m(AA^{\top}(I-P)).
\end{equation*}
This hybrid estimator typically achieves a significantly lower variance than the standard Hutchinson method, which in turn leads to better identification of the maximum row.
We follow the Hutch++ strategy and set $r = m / 3$, where $m$ is the whole budget for matrix-vector products (see ablation of different strategies for choosing $r$ in the supplement materials).
A detailed description of the algorithm is provided in \Cref{algo:twinest_plus_plus}. Below we provide the counterpart of \Cref{theorem:twinest_bound} for the TwINEst++ estimator.

\begin{algorithm}[t!]
\caption{TwINEst++}
\label{algo:twinest_plus_plus}
\begin{algorithmic}%
  \REQUIRE{
    \emph{}\\
    Matrix-vector multiplication oracle for \( A \in \mathbb{R}^{d \times n} \), \\
    Matrix-vector multiplication oracle for \( A^T \in \mathbb{R}^{n \times d} \), \\
    Positive integer \( m \in \mathbb{N} \): total sampling budget.
  }
  \ENSURE{
    \emph{}\\
    An estimate of the $\|A\|_{2 \to \infty}$ norm.
  }
\end{algorithmic}%
\begin{algorithmic}[1]
\STATE Sample \( \frac{m}{3} \) random Rademacher vectors $X^1$, $\dots$, $X^{\frac{m}{3}}$, where each \( X^i \in \{-1, 1\}^d \)
\STATE Sample random Rademacher matrix \( S \in \mathbb{R}^{d \times \frac{m}{3}} \), with i.i.d. \( \{-1, 1\}\) entries
\STATE Compute an orthonormal basis \( Q \) for \(AA^\top S\)\\
\COMMENT{via QR decomposition}
\FOR{each \( i = 1, 2, \dots, \frac{m}{3} \)}
    \STATE Compute \( t_i = X^i \odot AA^\top(I - QQ^\top) X^i \)
\ENDFOR
\STATE Compute \( \hat{D} = \frac{3}{m} \sum_{i=1}^{m/3} t_i \in \rset^{d} \) 
\STATE Compute \( D = \hat{D} + \diag(AA^\top QQ^\top) \)\\
\COMMENT{\(D\) - estimate of the \(AA^\top\) diagonal}
\STATE Find \( j = \arg\max_{\mathrm{i}} D_\mathrm{i} \)
\STATE Compute \( L = \| A^\top e_j \|_2 \)\\
\COMMENT{\(e_j\) - is the \(j\)-th standard basis vector}
\RETURN \( L \)
\end{algorithmic}
\end{algorithm}

\begin{figure*}[!t]
    \centering
    \begin{subfigure}[l]{0.85\columnwidth}
        \centering
        \includegraphics[width=\textwidth]{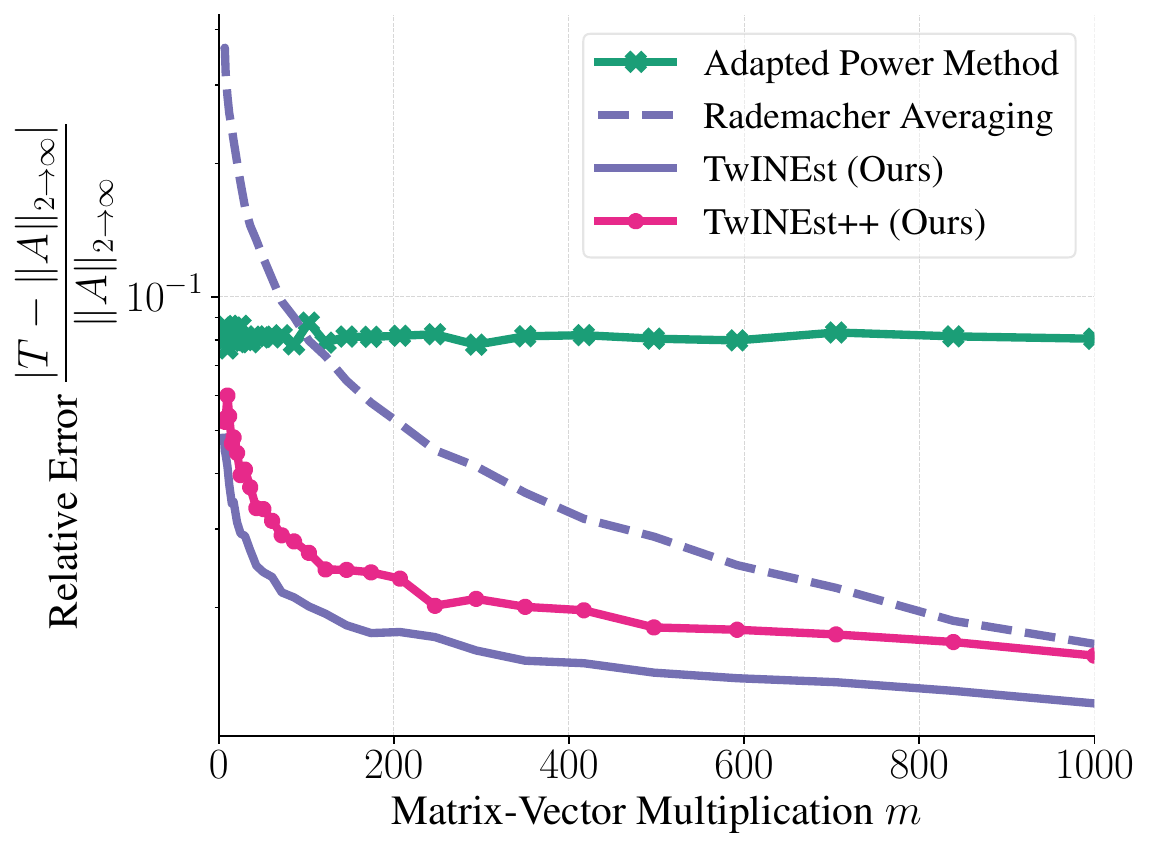}
        \subcaption{Synthetic data with $\Delta = 10^{-2}$.}
    \end{subfigure}\hfill
    \begin{subfigure}[l]{0.85\columnwidth}
        \centering
        \includegraphics[width=\textwidth]{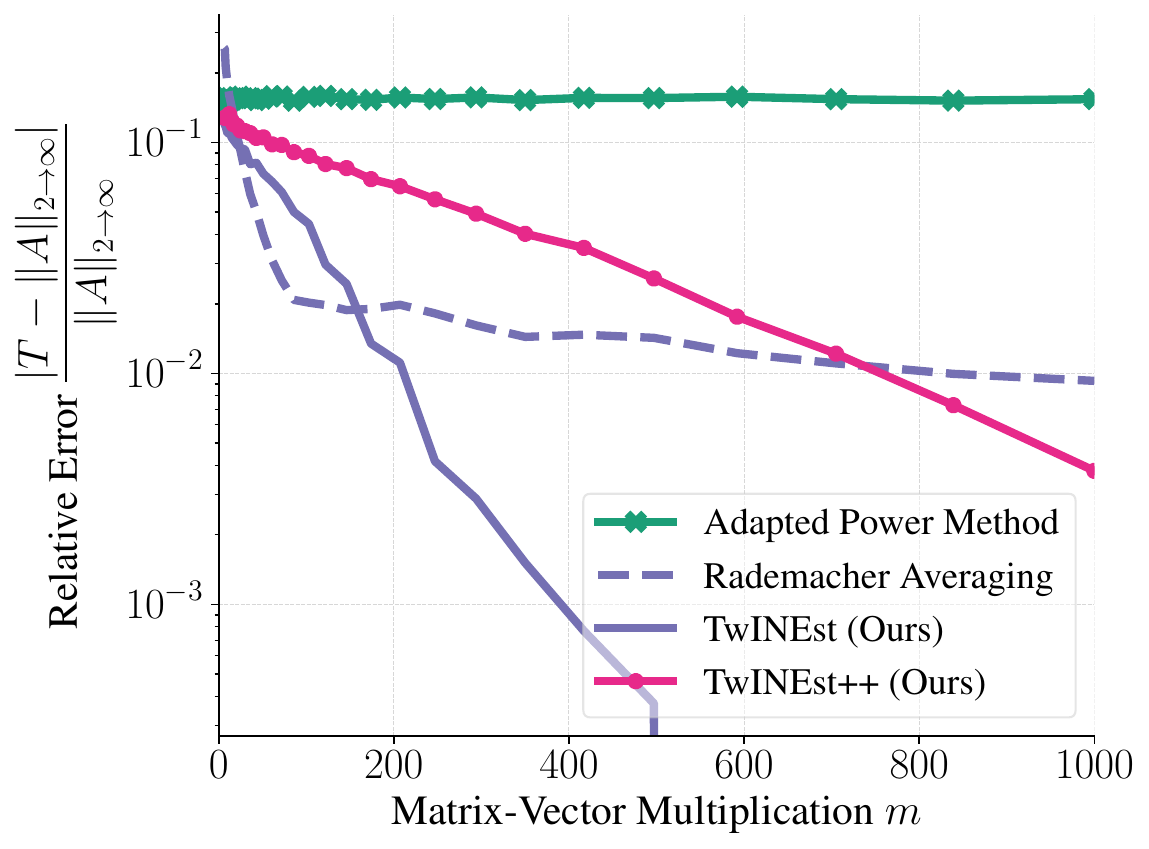}
        \subcaption{Synthetic data with $\Delta = 10^{-1}$.}
    \end{subfigure}
    \caption{Comparison of methods for estimating the two-to-infinity norm on random square matrices. Shown is the relative error versus the number of matrix-vector multiplications, averaged over 500 trials.}
    \label{fig:exps_synthetic}
\end{figure*}

\begin{theorem}[TwINEst++ Oracle Complexity]
\label{theorem:twinest_pp_bound}
Let $A \in \mathbb{R}^{d \times n}$, $m \in \mathbb{N}$ and $c \in \rset_+$ -- some constant. Let $T^m_{++}(A)$ be the output of Algorithm~\ref{algo:twinest_plus_plus} based on $m$ samples. Then, it suffices to choose
\begin{equation*}
m > c \cdot \left( \frac{\sqrt{\log(2/\delta)}}{\Delta} \|A\|_F^2 + \log(1/\delta) \right)
\end{equation*}
to ensure that $T^m_{++}(A) = \|A\|_{2 \to \infty}$ with probability at least $1 - \delta$.
\end{theorem}

\begin{proof}
The structure of the proof follows the same reasoning as in the case of \Cref{theorem:twinest_bound} for the base TwINEst algorithm. In particular, we again rely on a concentration inequality for the diagonal estimator. However, for TwINEst++, we employ a refined concentration result provided in \Cref{theorem:tight_diagonal_bound}, which yields a tighter control over the estimation error and enables the improved complexity result stated below.
Full proof is provided in the supplement paper.
\end{proof}

\Cref{theorem:twinest_pp_bound} shows that TwINEst++ achieves an improved oracle complexity compared to the original algorithm, particularly in challenging scenarios when $\Delta \to 0$, making identification of the correct row difficult. Specifically, the oracle complexity is reduced from $O(1/\Delta^2)$ in the original TwINEst algorithm to $O(1/\Delta)$ in TwINEst++.
Moreover, we expect that TwINEst++ performs particularly well in low-rank settings, exploiting the low-rank structure of the matrices, as in recent methods \cite{meyer2021hutch++,parkina2025coala}.
However, TwINEst’s oracle complexity depends on the norm of the matrix without its diagonal, while for TwINEst++ it depends on the norm of the whole matrix, so TwINEst may perform better for some matrices.

\section{Experiments}
\label{sec:experiments}

In this section, we present an empirical evaluation of the proposed algorithms. We focus on comparing the algorithms on synthetic and real-world matrices (see \Cref{exp:synthetic} and \Cref{exp:real_world}), and applications of our methods to image classification and recommender systems tasks (see \Cref{exp:dl_application} and \Cref{exp:recsys}).
All experiments are conducted on a single NVIDIA Tesla V100 GPU with 32GB memory.

\subsection{Synthetic Data}
\label{exp:synthetic}

In the following two sections, we compare the following methods:
\begin{itemize}
    \item \textbf{Adaptive Power Method.} A modification of the power iteration method for estimating the two-to-infinity norm (see \Cref{algo:adaptive_power_method} in the supplementary paper), based on \cite{higham1992estimating,roth2020adversarial}.
    \item \textbf{Rademacher Averaging.} A version of the TwINEst method that does not compute the exact norm of the candidate row with the maximum norm. Formally, for a matrix $A\in\mathbb{R}^{d \times n}$ we output $\sqrt{\max_i D^m_i(AA^\top)}$ (see notations in \Cref{algo:twinest}).
    \item \textbf{TwINEst.} Our algorithm introduced in \Cref{algo:twinest}. 
    \item \textbf{TwINEst++.} A variance-reduced version of TwINEst, described in \Cref{algo:twinest_plus_plus}.
\end{itemize}
The Adaptive Power Method lacks theoretical guarantees and may diverge on certain matrices (see \Cref{example:divergence_power_method}). Therefore, we hypothesize that our algorithms will outperform it.  Experiments with Rademacher Averaging serve as a natural ablation study for the TwINEst method.

For our synthetic experiments, we generate random Gaussian matrices $A \in \mathbb{R}^{5000\times5000}$. Specifically, we fix a parameter $\Delta \in (0, 1)$ and sample values $c_3, \dots, c_{5000} \sim \mathcal{U}[0,1]$, setting $c_1 = 1 + \Delta$, $c_2 = 1$. Each row of the matrix $A$ is normalized such that its squared $\ell_2$-norm is equal to $c_i$, ensuring a gap of magnitude $\Delta$ between the largest and second-largest squared row norms. The singular value distributions of these matrices are quite similar for different values of $\Delta$ (see \Cref{fig:exp_sv_delta} in the supplementary paper). Therefore, the parameter $\Delta$ may have a significant impact on the convergence rate.

The results of synthetic experiments for different values of $\Delta$ are illustrated in \Cref{fig:exps_synthetic}. As anticipated, TwINEst consistently outperforms Rademacher Averaging. The Adaptive Power Method fails to converge, as evidenced by its flat performance line. For a relatively large gap $\Delta = 10^{-1}$, the TwINEst algorithm rapidly converges, achieving accurate results consistently within approximately 400 iterations.

\subsection{WideResNet Jacobian}
\label{exp:real_world}

\begin{figure}[!t]
    \centering\includegraphics[width=0.8\columnwidth]{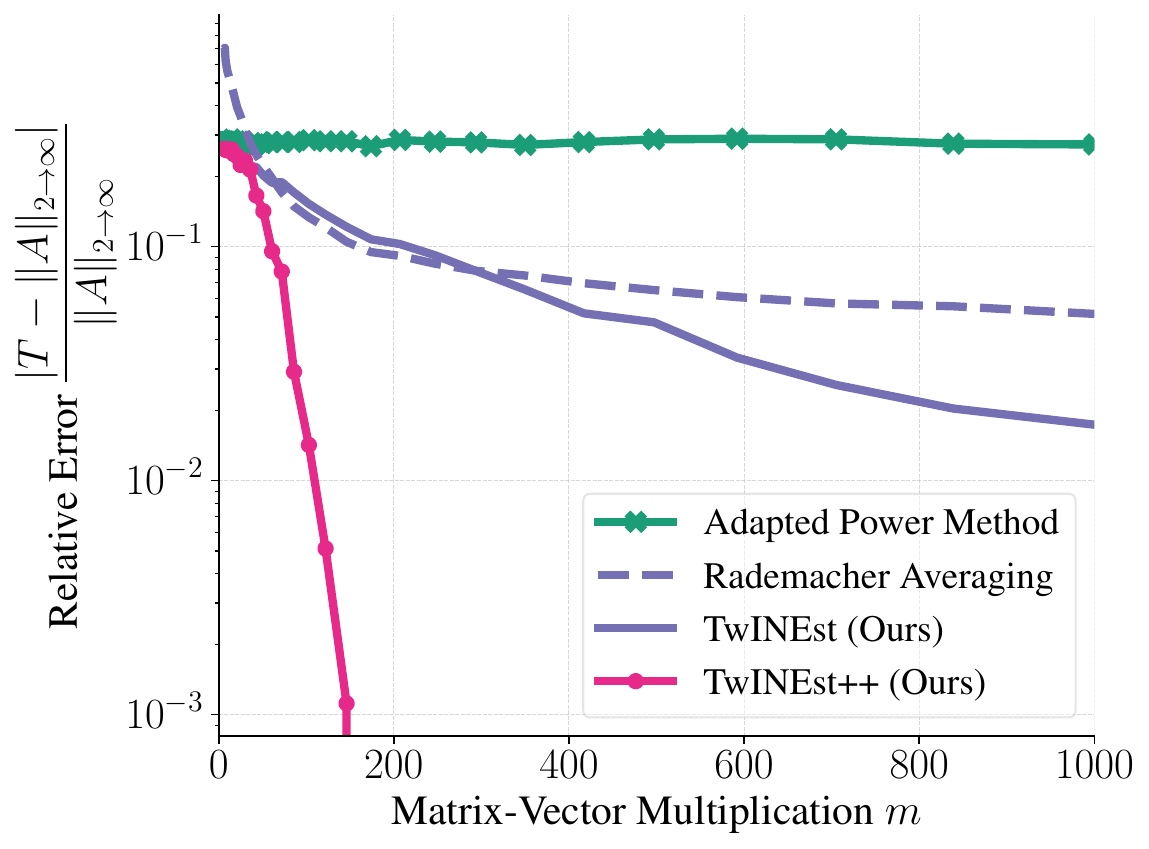}
    \caption{Comparison of methods for estimating the two-to-infinity norm of the Jacobian matrix of WideResNet-16-10 trained on CIFAR-100. The plot shows the relative error versus the number of matrix-vector multiplications, averaged over 500 trials.}
    \label{fig:exp_wideresnet}
\end{figure}

To validate our methods on real-world data, we evaluate them on the Jacobian matrix $J \in \mathbb{R}^{3 \cdot 32 \cdot 32 \times 100}$ of a WideResNet-16-10 \cite{zagoruyko2016wide} trained on CIFAR-100 \cite{krizhevsky2009learning}. Given the structure of $J$ (the number of columns is much smaller than the number of rows), we expect that the $AA^{\top}P$ term (see \Cref{eq:low_rank_decomposition}) in TwINEst++ will capture almost the entire span of $AA^{\top}$ with high probability. Therefore, TwINEst++ should outperform the basic TwINEst method on this matrix.

\Cref{fig:exp_wideresnet} confirms our hypothesis: the TwINEst++ algorithm achieves rapid convergence consistent with the rank of matrix $J$, significantly outperforming other algorithms. Again, the Adaptive Power Method fails to converge. For a comprehensive ablation, we provide a comparison between the relative error of the methods and their floating point operation counts (FLOPs) in \Cref{fig:flops}.

\subsection{Deep Learning Application}
\label{exp:dl_application}

\begin{table*}[th!]
\centering
\begin{tabular}{l|cccc|cccc}
\toprule
& \multicolumn{4}{c|}{\textbf{CIFAR-100}} 
  & \multicolumn{4}{c}{\textbf{TinyImageNet}} \\
\cmidrule{2-5} \cmidrule{6-9}
\textbf{Regularizer} & Acc.~\(\uparrow\) & FGSM~\(\uparrow\) & PGD~\(\uparrow\) & S. Rank~\(\downarrow\)
& Acc.~\(\uparrow\) & FGSM~\(\uparrow\) & PGD~\(\uparrow\) & S. Rank~\(\downarrow\) \\
\midrule
No regularization         & 75.5 $_{\pm0.2}$ & 24.4$_{\pm0.6}$ & 11.7 $_{\pm0.3}$ & 32.0 $_{\pm1.1}$ & 57.8 $_{\pm1.3}$ & 30.4$_{\pm0.3}$ & 20.2$_{\pm0.1}$ & 30.9 $_{\pm4.3}$ \\
Frobenius               &75.7 $_{\pm0.5}$ & 23.5$_{\pm0.2}$ & 13.3$_{\pm0.2}$ &31.6 $_{\pm0.2}$  & 58.6 $_{\pm0.3}$ & \bf31.1$_{\bf\pm0.2}$ & 20.7$_{\pm0.5}$ & 27.8 $_{\pm0.9}$ \\
Spectral                & 75.7 $_{\pm0.3}$ & 23.3$_{\pm0.7}$ & 11.3$_{\pm0.4}$ & 32.0 $_{\pm1.0}$  & 57.4 $_{\pm0.8}$ & 30.0$_{\pm1.1}$ & 20.0$_{\pm0.5}$ & 28.2 $_{\pm0.3}$ \\
Infinity           & 75.8 $_{\pm0.4}$ & 23.7$_{\pm0.7}$ & 11.1$_{\pm0.2}$ & 30.7 $_{\pm1.2}$ & 57.1 $_{\pm0.7}$  & 29.6$_{\pm1.4}$ & 19.7$_{\pm1.0}$ & 28.8 $_{\pm0.9}$  \\
\bf Two-to-Infinity (ours) & \bf 77.3 $_{\bf \pm0.1}$ & \bf 26.9$_{\bf\pm0.5}$ & \bf 14.5$_{\bf\pm0.5}$ & \bf 18.3 $_{\bf\pm0.8}$ & \bf 59.6 $_{\bf\pm0.9}$ & \underline{31.0$_{\pm1.0}$} & \bf23.4$_{\bf\pm0.7}$ & \bf 24.9 $_{\bf\pm0.3}$  \\
\bottomrule
\end{tabular}
\caption{Comparison of Jacobian regularization methods on CIFAR-100 and TinyImageNet datasets using WideResNet-16-10. Metrics are averaged over $3$ trials. Up arrow ($\uparrow$) indicates higher is better, while down arrow ($\downarrow$) indicates lower is better.}
\label{tab:jacobian}
\end{table*}

We study whether penalizing the two-to-infinity norm of the input-output Jacobian can improve the generalization ability and adversarial robustness of neural networks in image classification. Theorem 1 in \cite{roth2020adversarial} establishes a formal connection between $\ell_{p \to q}$ Jacobian norm regularization and adversarial training. This motivates our hypothesis that  two‑to‑infinity regularization can enhance adversarial robustness of deep neural networks.

We compare our regularizer to the standard Jacobian-based penalties: Frobenius norm \cite{hoffman2019robust}, spectral norm \cite{roth2020adversarial}, and infinity norm ($\ell_\infty$) \cite{roth2020adversarial}. For most image-classification datasets (e.g., CIFAR-100, TinyImageNet \cite{le2015tiny}), the number of output classes is much smaller than the number of input features (pixels). Consequently, the Jacobian of an image classifier trained on these datasets is a tall matrix. Hence, the two-to-infinity norm’s finer control over the Jacobian’s elements leads us to expect our regularizer to outperform standard norm penalties such as the Frobenius and spectral norms.

Formally, we minimize the following objective function:
\begin{equation*}
    \mathcal{L}(x, y) = \mathcal{L}_{\text{CE}}(f(x), y) + \lambda \cdot \|J_f(x)\|^2,
\end{equation*}
where $\mathcal{L}_{\text{CE}}$ is the cross-entropy loss, $x \in \rset^{d}$, $f(x) \in \rset^{M}$ is the output of a deep neural network, $M \in \mathbb{N}$ is the number of classes, $J_f(x) \in \rset^{d \times M}$ is the Jacobian of the logits $f(x)$ with respect to the input $x$, and $\|\cdot\|$ denotes one of the following norms: Frobenius, spectral, infinity, or two-to-infinity.

Explicit construction of the Jacobian of a deep network is computationally prohibitive in terms of both time and memory. Instead, we approximate each norm using the estimators described in the cited papers (see details in \Cref{appendix:reg_info}). For the two-to-infinity norm, we employ TwINEst algorithm. All of these estimators require only Jacobian–vector and vector–Jacobian products, which are supported by the automatic-differentiation frameworks. However, such operations cost roughly as much as a single backward pass, and naively applying the regularizer at every training step can slow down learning. To provide a more comprehensive ablation, we show that updating the regularization term once every $k$ iterations still outperforms other methods, while adding negligible wall-clock overhead relative to training without regularization. For more details, see \Cref{appendix:regularizer_ablation} in the supplementary paper.

Experiments are conducted on CIFAR-100 and Tiny\-Image\-Net datasets using the Wide\-Res\-Net-16-10 architecture, implemented in PyTorch. The hyperparameters are provided in \Cref{appendix:hyperparams}. We evaluate each method by reporting the final test accuracy, the stable rank of the Jacobian (computed as $\|J\|_F^2 / \|J\|_2^2$), and adversarial metrics: accuracy after FGSM \cite{goodfellow2015explaining} and PGD \cite{madry2018towards} attacks with $2$ steps.

As shown in \Cref{tab:jacobian}, our method enhances both the generalization performance and the adversarial robustness of WideResNet‑16‑10 on the CIFAR‑100 and TinyImageNet datasets. Other approaches show limited improvements over the baseline.

\subsection{RecSys Application}
\label{exp:recsys}

\begin{figure*}[ht]
    \centering
    \begin{subfigure}{0.32\textwidth}
        \centering
        \includegraphics[width=\textwidth]{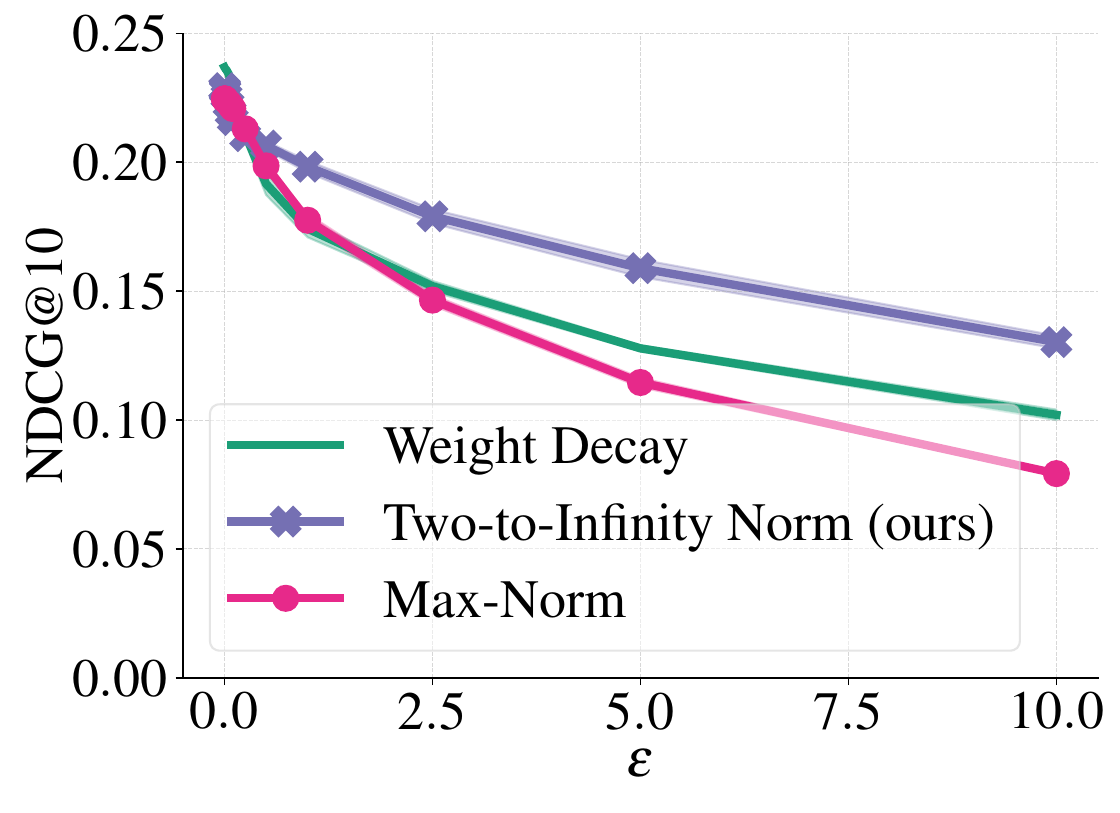}
        \subcaption{MovieLens-1M dataset.}
    \end{subfigure}\hfill
    \begin{subfigure}{0.32\textwidth}
        \centering
        \includegraphics[width=\textwidth]{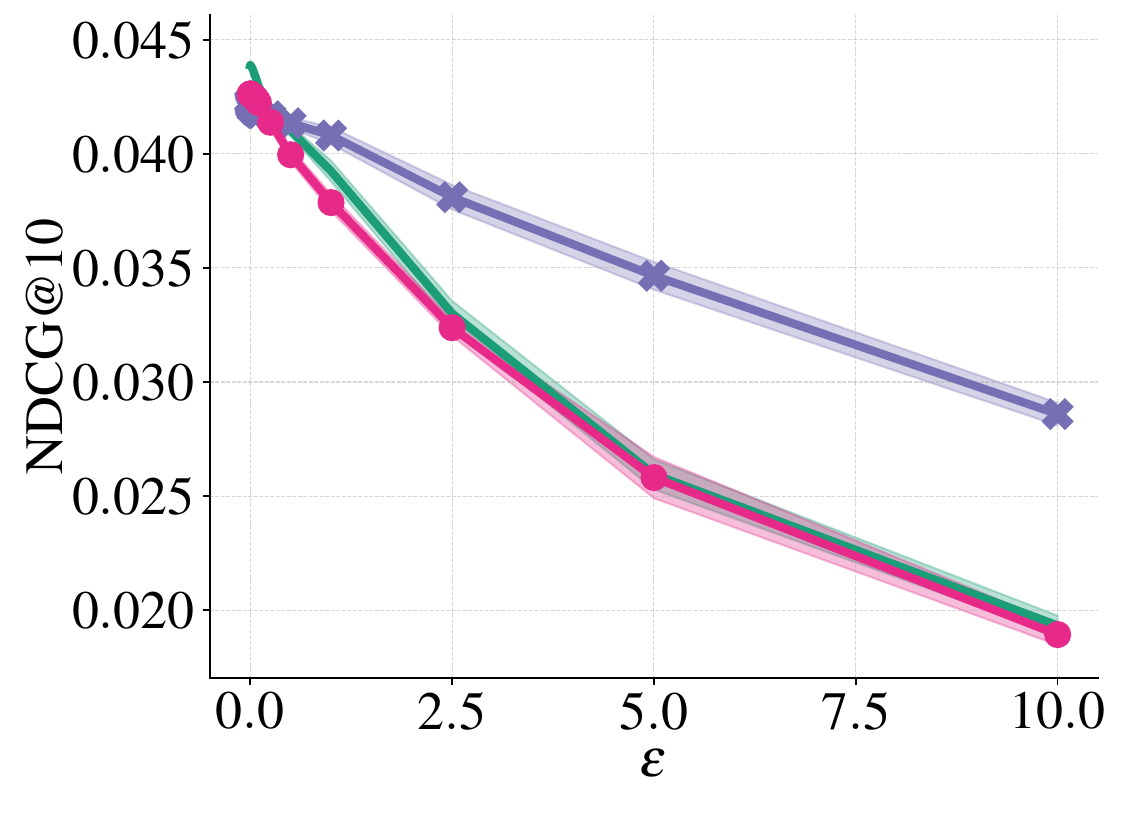}
        \subcaption{Yelp2018 dataset.}
    \end{subfigure}\hfill
    \begin{subfigure}{0.32\textwidth}
        \centering
        \includegraphics[width=\textwidth]{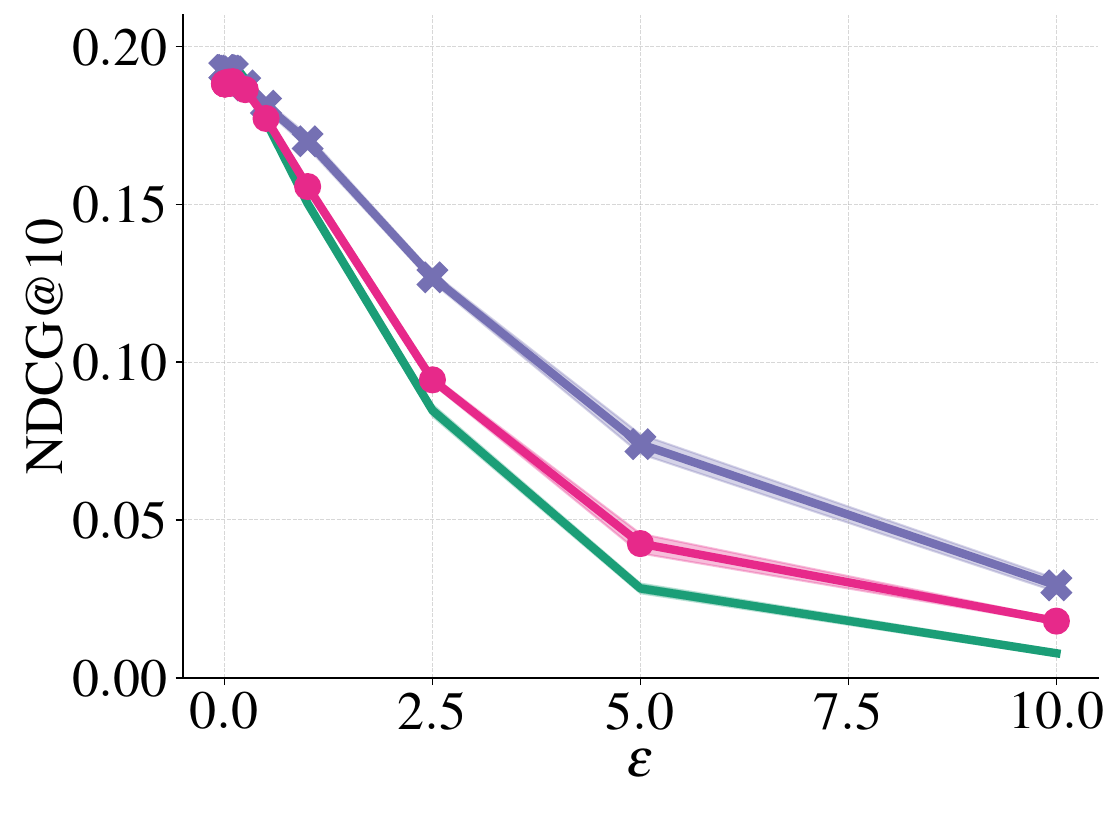}
        \subcaption{CiteULike dataset.}
    \end{subfigure}
    \caption{Comparison of different regularization methods for adversarial robustness of UltraGCN \cite{mao2021ultragcn}. Shown is the NDCG@10 metric(higher is better) versus the magnitude of the attack. Metric is averaged over 5 trials.}
    \label{fig:exp_recsys}
\end{figure*}

In this section, we study the application of our two-to-infinity norm regularization to improve the adversarial robustness of recommender systems. We focus on the collaborative filtering task with $d_u$ users and $d_i$ items, and use the classic matrix factorization framework:
\begin{equation}
\label{eq:matrix_decomp}
    R \approx UV^\top,
\end{equation}
where $R \in \mathbb{R}^{d_u \times d_i}$ is the user-item interaction matrix, and $U \in \mathbb{R}^{d_u \times r}$, $V \in \mathbb{R}^{d_i \times r}$ are users and items embedding matrices, respectively. The predicted interaction scores form the matrix $\hat{R} = UV^\top \in \mathbb{R}^{d_u \times d_i}$, where $\hat{R}_{u,i}$ denotes the predicted score for the interaction between the $u$-th user and the $i$-th item.

Some collaborative filtering models (e.g., \cite{rendle2009bpr, mao2021ultragcn}) can be expressed within the framework described in \Cref{eq:matrix_decomp}. We follow the adversarial setting introduced in \cite{he2018adversarial}, and adapt it to the UltraGCN model \cite{mao2021ultragcn}. The UltraGCN loss function is defined as:
\begin{equation*}
    \mathcal{L}(U, V) = \mathcal{L}_{\text{BCE}}(U, V) + \gamma \cdot \mathcal{L}_{\text{C}}(U, V) + \lambda \cdot (\|U\|_F^2 + \|V\|_F^2),
\end{equation*}
where \( \mathcal{L}_{\text{BCE}} \) is the binary cross-entropy loss, and $\mathcal{L}_{\text{C}}$ is a constraint loss defined in \cite{mao2021ultragcn}.

In the adversarial setting, the goal is to find a small perturbation of the model parameters that significantly degrades prediction quality. Since only the binary cross-entropy term directly affects recommendation accuracy, we aim to find perturbations that maximize this component:
\begin{equation*}
    \Delta_U^*, \Delta_V^* \in \argmax_{\|\Delta_U\|_2, \|\Delta_V\|_2 \leq \varepsilon} \mathcal{L}_{\text{BCE}}(U + \Delta_U, V + \Delta_V),
\end{equation*}
where \( \varepsilon \) controls the attack magnitude. A practical way to obtain such perturbations is via:
\begin{equation*}
    \Delta_U^{\text{adv}} = \varepsilon \cdot \frac{\Gamma_U}{\|\Gamma_U\|_2}, \quad \text{and} \quad \Delta_V^{\text{adv}} = \varepsilon \cdot \frac{\Gamma_V}{\|\Gamma_V\|_2},
\end{equation*}
where $\Gamma_U$ and $\Gamma_V$ are the gradients of $\mathcal{L}_{\text{BCE}}$ with respect to $U$ and $V$, respectively. To improve the adversarial robustness of the UltraGCN model, we modify its loss function by replacing the standard weight decay term ($\|U\|_F^2 + \|V\|_F^2$) with a regularization term based on the two-to-infinity norm of the score matrix. This regularizer provides better control over the embedding matrices by taking into account the impact of every user embedding and penalizing the largest ones. Intuitively, this should prevent the model from overfitting to a small subset of popular users. Specifically, our loss function becomes:
\begin{equation*}
    \mathcal{L}_{\text{ours}}(U, V) = \mathcal{L}_{\text{BCE}}(U, V) + \gamma \cdot \mathcal{L}_{\text{C}}(U, V) + \lambda \cdot \|\hat{R}\|_{2 \to \infty}^2.
\end{equation*}
Note that computing the full score matrix $\hat{R}$ at every training iteration is computationally prohibitive due to its large size and memory requirements. To address this, we approximate the two-to-infinity norm using our TwINEst algorithm.
We compare our two-to-infinity regularizer with the standard weight decay used in the UltraGCN model, as well as with the two-to-infinity norm of the factors ($\|U\|_{2\to\infty}^2 + \|V\|_{2\to\infty}^2$), which is closely related to the max-norm regularizer popular in recommender systems \cite{srebro2004maximum}.

Experiments are conducted on the MovieLens-1M \cite{harper2015movielens}, Yelp-2018 \cite{asghar2016yelp}, and CiteULike \cite{wang2013collaborative} datasets, using the UltraGCN architecture implemented in PyTorch. The hyperparameter configurations and the dataset preprocessing pipeline are provided in \Cref{appendix:hyperparams_recsys}. We evaluate methods by reporting the NDCG@10 metric (which rewards ranking relevant items near the top) \cite{jarvelin2002cumulated}. Additional results using other metrics are provided in \Cref{appendix:add_metrics_recsys}.

Results are shown in \Cref{fig:exp_recsys}. On every dataset, our regularization demonstrates improved robustness under adversarial attacks of moderate to high magnitude. While our method achieves comparable metrics to the baselines under no attack, it exhibits greater stability under strong perturbations.

\section{Conclusion}
\label{sec:conclusion}

In this paper, we propose two novel matrix-free stochastic algorithms for estimating the two-to-infinity and one-to-two norms, and provide a theoretical analysis of their behavior. Our empirical results demonstrate that the proposed methods outperform existing approaches in terms of both accuracy and computational efficiency. Furthermore, we show that our algorithms can be easily integrated into deep learning pipelines and applied to improve robustness to adversarial attacks in recommender systems. Possible directions for future research include establishing lower bounds on the sample and oracle complexity of the two-to-infinity norm estimation or exploring additional applications of two-to-infinity norm estimation in the matrix-free setting.

\section*{Acknowledgements}
The work was supported by the grant for research centers in the field of AI provided by the Ministry of Economic Development of the Russian Federation in accordance with the agreement 000000C313925P4E0002 and the agreement with HSE University № 139-15-2025-009. This research was supported in part through computational resources of HPC facilities at HSE University \cite{kostenetskiy2021hpc}.

\bibliography{aaai2026}

\onecolumn

\appendix
\label{appendix}

\section{Technical Lemmas}
\label{appendix:tech_lemmas}

\begin{lemma}
\label{lemma:diag_one_element}
Let $A \in \mathbb{R}^{d \times d}$, $m \in \mathbb{N}$, $i \in [d]$, and $\varepsilon \geq 0$. Then
\begin{equation*}
    \mathbb{P}\left(|D^m_i(A) - A_{ii}| \geq \varepsilon\right)
    \leq
    2 \exp\left(-\frac{\varepsilon^2 m}{2 (\|A_i\|_2^2 - A_{ii}^2)}\right).
\end{equation*}
\end{lemma}

\begin{proof}
\label{appendix:proof_diag_one_element}

This statement follows from Theorem 2 in~\cite{baston2022stochastic}, but we provide an independent argument. Since $(X^k_i)^2 = 1$ for Rademacher random variables,
\begin{equation*}
    D^m_i(A) - A_{ii} = \frac{1}{m} \sum\limits_{k=1}^{m} (X^k \odot AX^k)_i - A_{ii}
    = \frac{1}{m} \sum\limits_{k=1}^{m} \left( A_{ii} (X^k_i)^2 + \sum_{j \neq i} X^k_i A_{ij} X^k_j \right) - A_{ii}
    = \frac{1}{m} \sum\limits_{k=1}^{m} \sum_{j \neq i} X^k_i A_{ij} X^k_j.
\end{equation*}
Define $Y^k_j = X^k_i X^k_j$. Since the product of two independent Rademacher variables is again Rademacher, and they remain mutually independent,
\begin{equation*}
    \mathbb{P}\left(|D^m_i(A) - A_{ii}| \geq \varepsilon\right)
    =
    \mathbb{P}\left(\left|\sum\limits_{k=1}^{m} \sum_{j \neq i} \frac{A_{ij}}{m} Y^k_j\right| \geq \varepsilon\right).
\end{equation*}
Applying Hoeffding’s inequality (see~\cite{vershynin2018high}) yields the desired result.
\end{proof}

\begin{lemma}
\label{lemma:concentration_inf_norm}
Let $A \in \mathbb{R}^{d \times d}$, $m \in \mathbb{N}$, $\varepsilon \geq 0$, and let $\bar{A}$ be the matrix $A$ with diagonal entries set to zero. Then
\begin{equation*}
    \mathbb{P}\left(\|D^m(A) - \diag(A)\|_{\infty} \geq \varepsilon\right)
    \leq
    2d \exp\left(-\frac{\varepsilon^2 m}{2 \|\bar{A}\|_{2 \to \infty}^2}\right).
\end{equation*}
\end{lemma}

\begin{proof}
\label{appendix:proof_concentration_inf_norm}

\begin{align*}
    \mathbb{P}\left(\|D^m(A) - \diag(A)\|_{\infty} \geq \varepsilon\right)
    &= \mathbb{P}\left(\max_i |D^m_i(A) - \diag(A)_i| \geq \varepsilon\right)\\
    &\leq \sum\limits_{i=1}^{d} \mathbb{P}\left(|D^m_i(A) - \diag(A)_i| \geq \varepsilon\right) \tag{By the union bound} \\
    & \leq \sum\limits_{i=1}^{d} 2\exp\left(-\frac{\varepsilon^2 m}{2 \|\bar{A}_i\|_2^2}\right) \tag{By Lemma~\ref{lemma:diag_one_element}} \\
    & \leq
    2d \exp\left(-\frac{\varepsilon^2 m}{2 \|\bar{A}\|_{2 \to \infty}^2}\right).
\end{align*}
\end{proof}

\begin{lemma}
\label{lemma:k_rank_approx}
Let $k \in \mathbb{N}$ and let $A \in \mathbb{R}^{d \times d}$ be a positive semidefinite (PSD) matrix. Denote by $A_k = \arg\min_{B: \rank(B) \leq k} \| A - B \|_F$ the best rank-$k$ approximation of $A$ in the Frobenius norm. Then,
\begin{equation*}
    \| A - A_k \|_F \leq \frac{1}{\sqrt{k}} \trace(A).
\end{equation*}
\end{lemma}

\begin{proof}
\label{appendix:proof_k_rank_approx}

Since $A$ is PSD, it admits an eigenvalue decomposition $A = U \Lambda U^\top$ with non-negative eigenvalues $\lambda_1 \geq \lambda_2 \geq \cdots \geq \lambda_d \geq 0$. The best rank-$k$ approximation $A_k$ is obtained by keeping the top $k$ eigenvalues. Therefore,
\begin{equation*}
\|A - A_k\|_F^2 = \sum_{i = k+1}^d \lambda_i^2.
\end{equation*}
Applying the inequality $\lambda_i \leq \lambda_{k+1}$ for all $i > k$ and using Cauchy–Schwarz, we obtain
\begin{equation*}
\sum_{i = k+1}^d \lambda_i^2 \leq \lambda_{k+1} \sum_{i=k+1}^d \lambda_i \leq \frac{1}{k} \left( \sum_{i=1}^d \lambda_i \right)^2 = \frac{1}{k} \trace(A)^2.
\end{equation*}
Taking the square root gives the desired bound.
\end{proof}

\section{High Dimensional Proofs}
\label{appendix:hdp}
\subsection{Proof of Theorem~\ref{theorem:twinest_bound}}
\label{appendix:proof_twinest_bound}

For simplicity of notation, let $B := AA^{\top}$ and $\bar{B}$ denote the matrix $B$ with its diagonal entries set to zero. Let $D := D^m(B)$ be the diagonal estimate of $B$.

Recall that, as discussed in Section~\ref{sec:main_algo}, the goal of the algorithm is to find an index corresponding to a row of maximal $\ell_2$-norm. The key observation is that for any $\gamma \in \arg\max_i B_{ii}$ (there might be multiple rows with maximal norm), we need to show that its estimate $D_\gamma$ dominates all other estimates $D_j$ for $j \in S$, where $S := \{i \mid i \notin \arg\max_i B_{ii}\}$ be the set of non-maximal rows.

By \Cref{lemma:concentration_inf_norm}, with probability at least $1 - \delta$:
\begin{equation*}
    \| D - \diag(B) \|_{\infty} \leq \varepsilon, \quad \text{ where } \quad \varepsilon = \sqrt{\frac{2 \log (2d / \delta)}{m}} \|\bar{B}\|_{2 \to \infty}.
\end{equation*}

This bound implies that for each $i$, $D_i \in [B_{ii} - \varepsilon, B_{ii} + \varepsilon]$ . Moreover, by definition of $\Delta$ for any $j$, $B_{\gamma\gamma} \geq B_{jj} + \Delta$. Combining these facts, we conclude that for any $\gamma \in \arg\max_i B_{ii}$,
\begin{equation*}
    D_\gamma - D_j 
    \geq (B_{\gamma\gamma} - \varepsilon) - (B_{jj} + \varepsilon) = \underbrace{(B_{\gamma\gamma} - B_{jj})}_{\geq \Delta} - 2\varepsilon
    > \Delta - 2\left(\Delta/2\right) = 0,
\end{equation*}
where the last inequality holds when $\varepsilon < \Delta/2$.

This shows that for any maximal row $\gamma$ and any non-maximal row $j$, $D_\gamma > D_j$ with probability at least $1 - \delta$. Therefore, the algorithm correctly identifies a maximal row, meaning that  $T^m(A) = \|A\|_{2 \to \infty}$.

Finally, the condition $\varepsilon < \Delta/2$ is equivalent to
\begin{equation*}
    m > \frac{8\log (2d / \delta)}{\Delta^2} \|\bar{B}\|_{2 \to \infty}^2
    =
    \frac{8\log (2d / \delta)}{\Delta^2} \|AA^{\top} - \diag(AA^{\top})\|_{2 \to \infty}^2
    .
\end{equation*}

\subsection{Proof of Theorem~\ref{theorem:twinest_pp_bound}}
\label{appendix:proof_twinest_pp_bound}

Define $B := AA^\top$. Let $k \in \mathbb{N}$, $l = c_1 \cdot (k + \log(1/\delta))$, where $c_1$ -- sufficiently large universal constant. Let $S \in \mathbb{R}^{d \times l}$ be a random Rademacher matrix, and let $Q$ be an orthonormal basis for the range of $BS$. We decompose $B$ as
\begin{equation*}
B = BQQ^\top + B(I - QQ^\top),
\end{equation*}
where $BQQ^\top$ can be computed exactly using $l$ matrix-vector products with $B$, and the challenge is to estimate $\operatorname{diag}(B(I - QQ^\top))$.

Define $\hat{D} := D^k(B(I - QQ^\top))$. Then we have with probability at least $1 - \delta$:

\begin{align*}
    \| \hat{D} - \diag(B(I - QQ^\top)) \|_{\infty}
    &\leq \| \hat{D} - \diag(B(I - QQ^\top)) \|_{2}\\
    & \leq c_2 \sqrt{\frac{\log(2/\delta)}{k}} \|B(I - QQ^T) \|_F \tag{By Theorem~\ref{theorem:tight_diagonal_bound}}\\
    & \leq 2c_2 \sqrt{\frac{\log(2/\delta)}{k}} \|B - B_k \|_F \tag{By Corollary 7 and Claim 1 from~\cite{musco2020projection}}\\
    & \leq 2c_2 \sqrt{\frac{\log(2/\delta)}{k^2}} \trace(B) \tag{By Lemma~\ref{lemma:k_rank_approx}}\\
    & = 2c_2 \sqrt{\frac{\log(2/\delta)}{k^2}} \|A\|_F^2 \tag{since $B = AA^\top$}\\
\end{align*}

Setting
$k > 4c_2 \frac{\sqrt{ \log(2/\delta) }}{ \Delta } \|A\|_F^2 $
ensures that
\begin{equation*}
\| \hat{D} - \operatorname{diag}(B(I - QQ^\top)) \|_{\infty} < \Delta / 2.
\end{equation*}
Finally, following the same reasoning as in the proof of Theorem~\ref{theorem:twinest_bound}, we conclude that  $T^m_{++}(A) = \| A \|_{2 \to \infty}$ with probability at least $1 - \delta$, when
\begin{equation*}
m = 2l + k > c \cdot \left( \frac{\sqrt{ \log(2/\delta) }}{ \Delta } \|A\|_F^2 + \log(1/\delta) \right),
\end{equation*}
for some sufficiently large constant $c$.

\section{Adaptive Power Method}
\label{appendix:adaptive_power_method}

\begin{algorithm}[H]
\caption{Adaptive Power Method for Two-to-Infinity Norm from \cite{higham1992estimating,roth2020adversarial}}
\label{algo:adaptive_power_method}
\begin{algorithmic}%
  \REQUIRE{
    \emph{}\\
    Oracle for matrix-vector multiplication with matrix \( A \in \mathbb{R}^{d \times n} \), \\
    Oracle for matrix-vector multiplication with matrix \( A^T \in \mathbb{R}^{n \times d} \), \\
    Positive integer \( m \in \mathbb{N} \): number of iterations.
  }
  \ENSURE{
    \emph{}\\
    An estimate of the $\|A\|_{2 \to \infty}$ norm.
  }
\end{algorithmic}%
\begin{algorithmic}[1]
\STATE Sample random vector \( X^0 \in \mathbb{R}^n \) from \( \mathcal{N}(0, I_n) \)
\FOR{each \( i = 1, 2, \dots, m \)}
    \STATE Compute \( Y^i = \operatorname{dual}_{\infty}(AX^{i - 1}) \)
    \STATE Compute \( X^i = \operatorname{dual}_2(A^\top Y^i) \)
\ENDFOR
\STATE Compute \( L = \| A X^m \|_{\infty} \)
\RETURN \( L \)
\end{algorithmic}
\end{algorithm}

\section{Details for Experiments}
\label{appendix:details_exps}

\subsection{Setting the Random Seed}
\label{appendix:set_random}

\lstset{style=mystyle}
\begin{lstlisting}[language=Python, caption=Python code used to fix the random seed.]
import os
import torch
import random
import numpy as np

seed = 42  # Random seed
torch.manual_seed(seed)
torch.backends.cudnn.deterministic = True
torch.backends.cudnn.benchmark = False
np.random.seed(seed)
random.seed(seed)
os.environ["PYTHONHASHSEED"] = str(seed)
\end{lstlisting}

\subsection{Singular Values of Evaluation Matrices}
\label{appendis:sv}

\Cref{fig:exp_sv_delta} shows the singular values of the synthetic matrices for different values of $\Delta$. Due to our matrix-generation scheme, the singular-value distributions are quite similar, whereas the real-world matrix in \Cref{fig:exp_sv_jac} exhibits a different distribution.

\begin{figure}[H]
\centering
    \begin{subfigure}{0.45\textwidth}
     \centering\includegraphics[width=\textwidth]{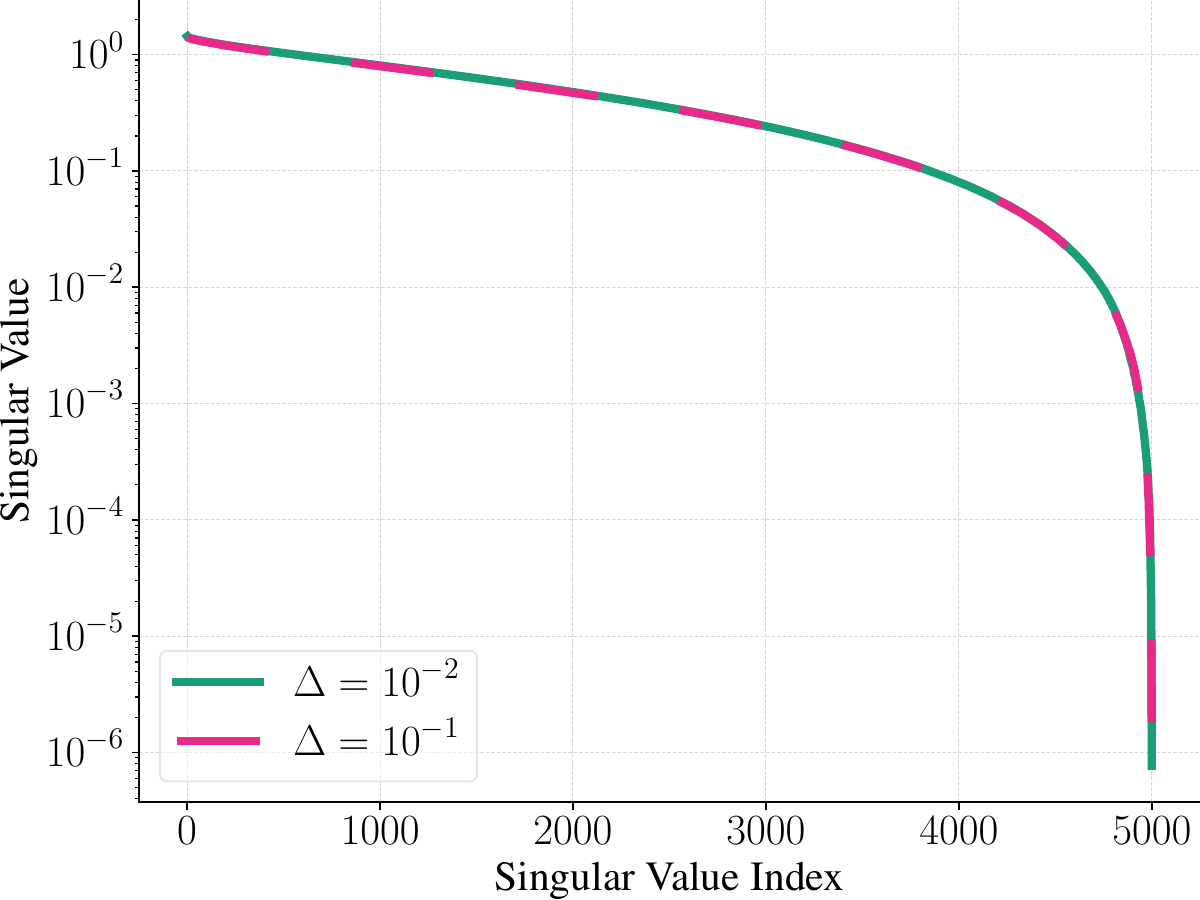}
    \subcaption{Singular values of synthetic matrices for different $\Delta$.}
    \label{fig:exp_sv_delta}
   \end{subfigure}\hfill
   \begin{subfigure}{0.45\textwidth}
     \centering\includegraphics[width=\textwidth]{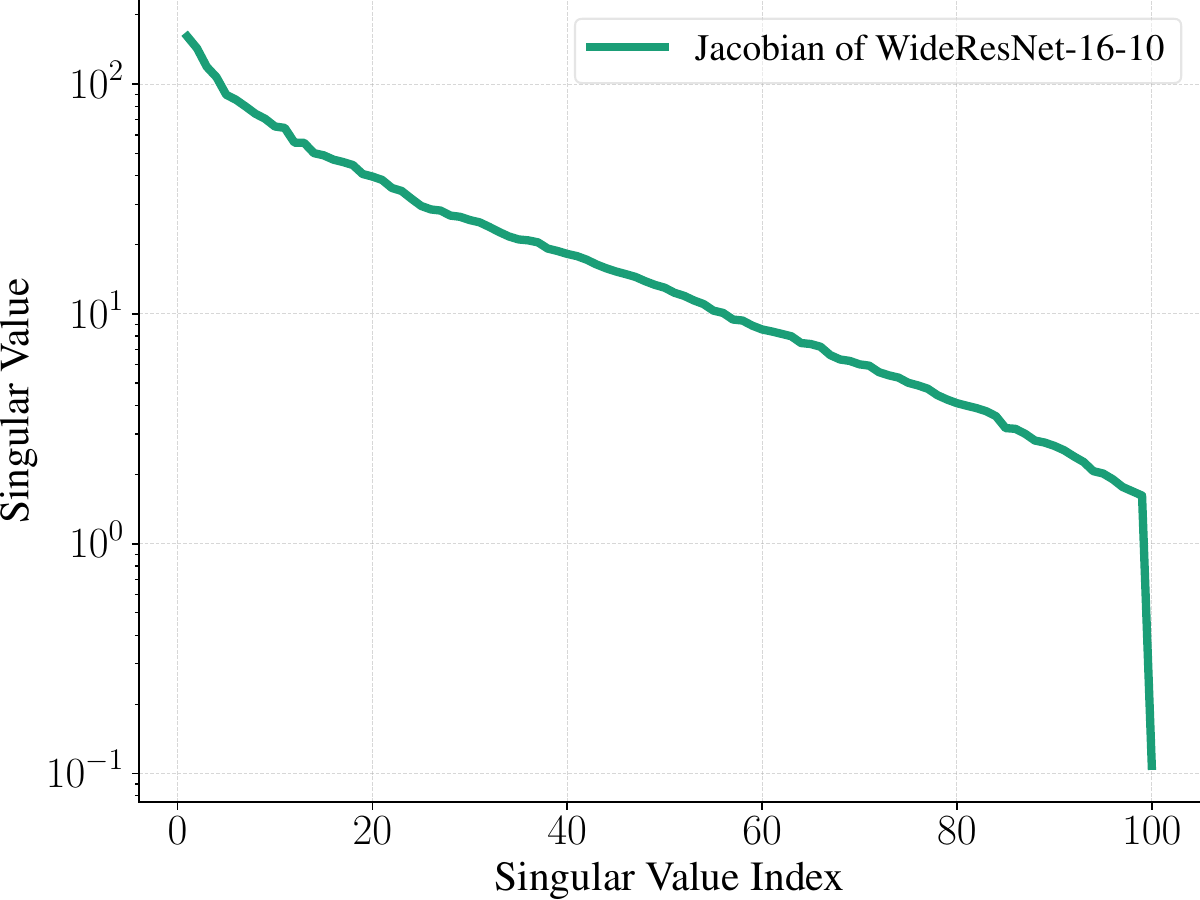}
    \subcaption{Singular values of the Jacobian matrix of WideResNet-16-10.}
    \label{fig:exp_sv_jac}
   \end{subfigure}
   \caption{Singular values of synthetic and real world matrices.}
   \label{fig:singular_values}
\end{figure}

\subsection{Ablation for Method Comparison}
\label{appendix:figures}

\Cref{fig:abl_r} presents an ablation study of different strategies for choosing $r$ (see \Cref{sec:improved_algo}). The base TwINEst++ algorithm uses a partition with $p = 1/3$ ($r = p \cdot m = m / 3$, where $m$ is the whole sampling budget). When $p \to 0$, TwINEst++ behaves like TwINEst, as expected, since the low-rank term has a negligible effect. When the low-rank term can capture the dominant eigenspace, as in \Cref{fig:abl_r:3}, it is better to use a greater $p$, as fewer terms will need to be approximated via the stochastic algorithm. Thus, $p$ represents a trade-off between not deviating too far from TwINEst in unfavorable settings and preserving most of the dominant eigenspace in the low-rank case.

\begin{figure}[H]
    \begin{subfigure}{0.3\textwidth}
     \centering\includegraphics[width=\textwidth]{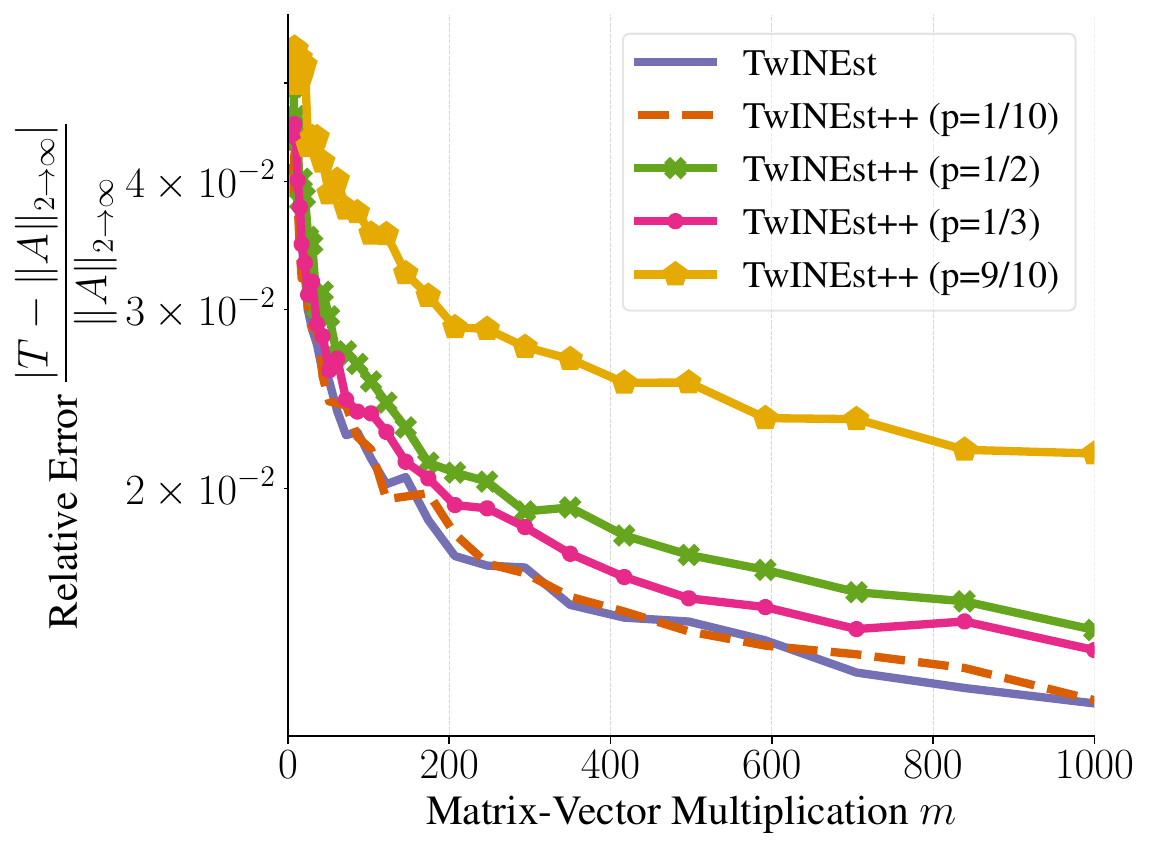}
    \subcaption{Synthetic data with $\Delta=10^{-2}$.}
   \end{subfigure}\hfill
   \begin{subfigure}{0.3\textwidth}
     \centering\includegraphics[width=\textwidth]{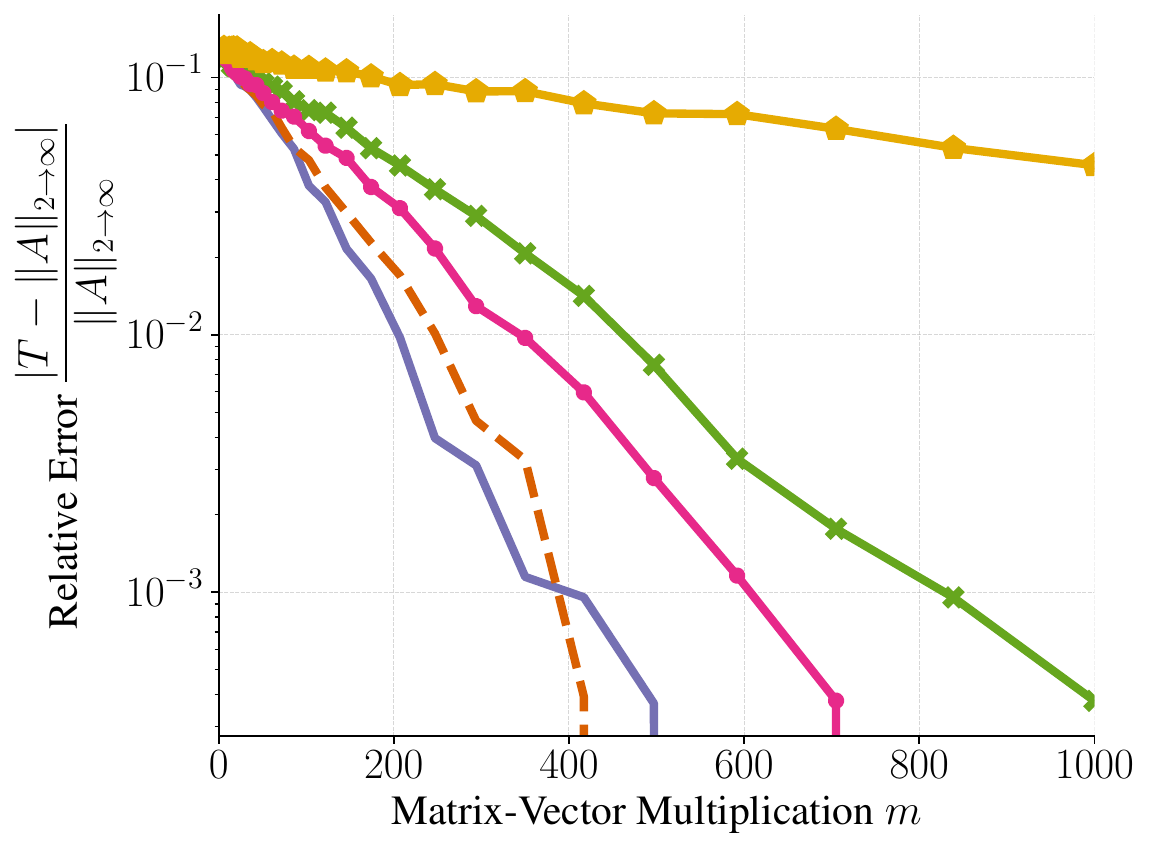}
    \subcaption{Synthetic data with $\Delta=10^{-1}$.}
   \end{subfigure}\hfill
   \begin{subfigure}{0.3\textwidth}
     \centering\includegraphics[width=\textwidth]{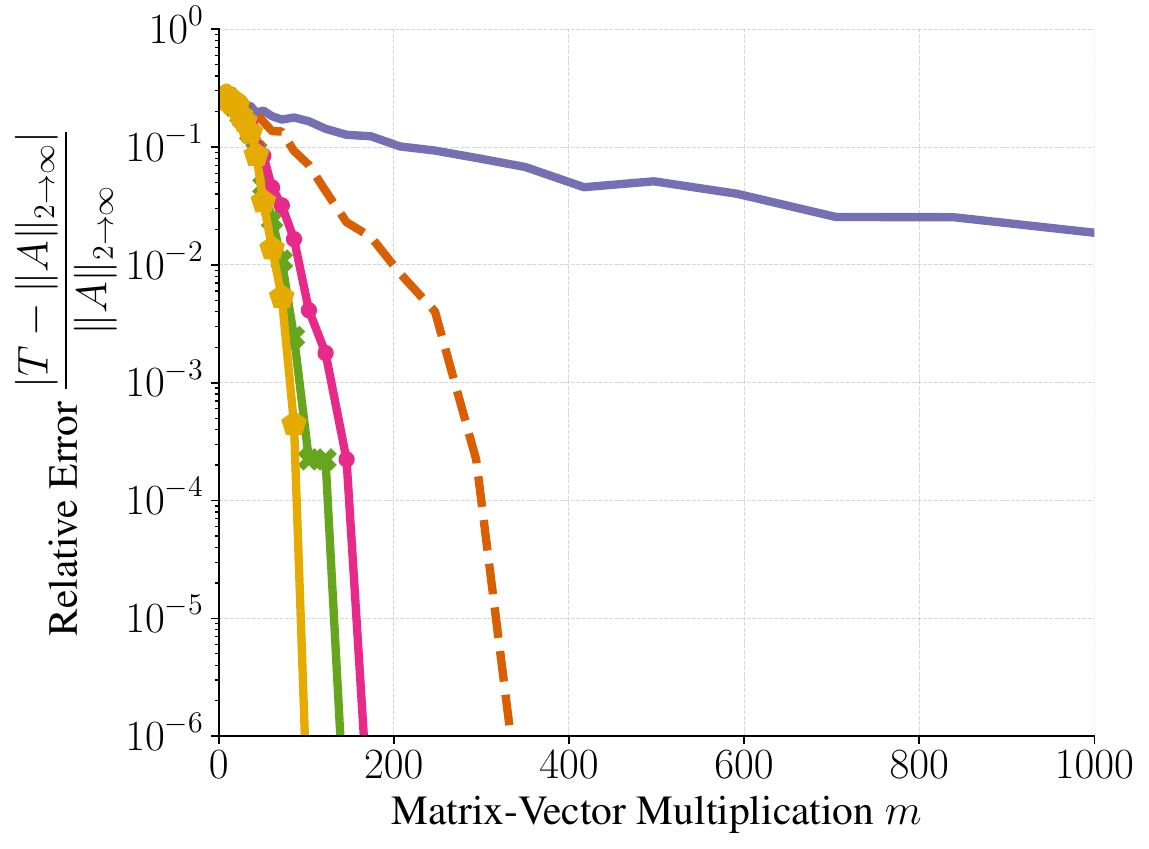}
    \subcaption{Jacobian matrix of WideResNet-16-10.}
    \label{fig:abl_r:3}
   \end{subfigure}
   \caption{Comparison of different strategies for choosing $r$ in TwINEst++ algorithm. The plot shows the relative error versus number of matrix-vector multiplications, averaged over 500 trials.}
   \label{fig:abl_r}
\end{figure}

\begin{figure}[H]
    \begin{subfigure}{0.3\textwidth}
     \centering\includegraphics[width=\textwidth]{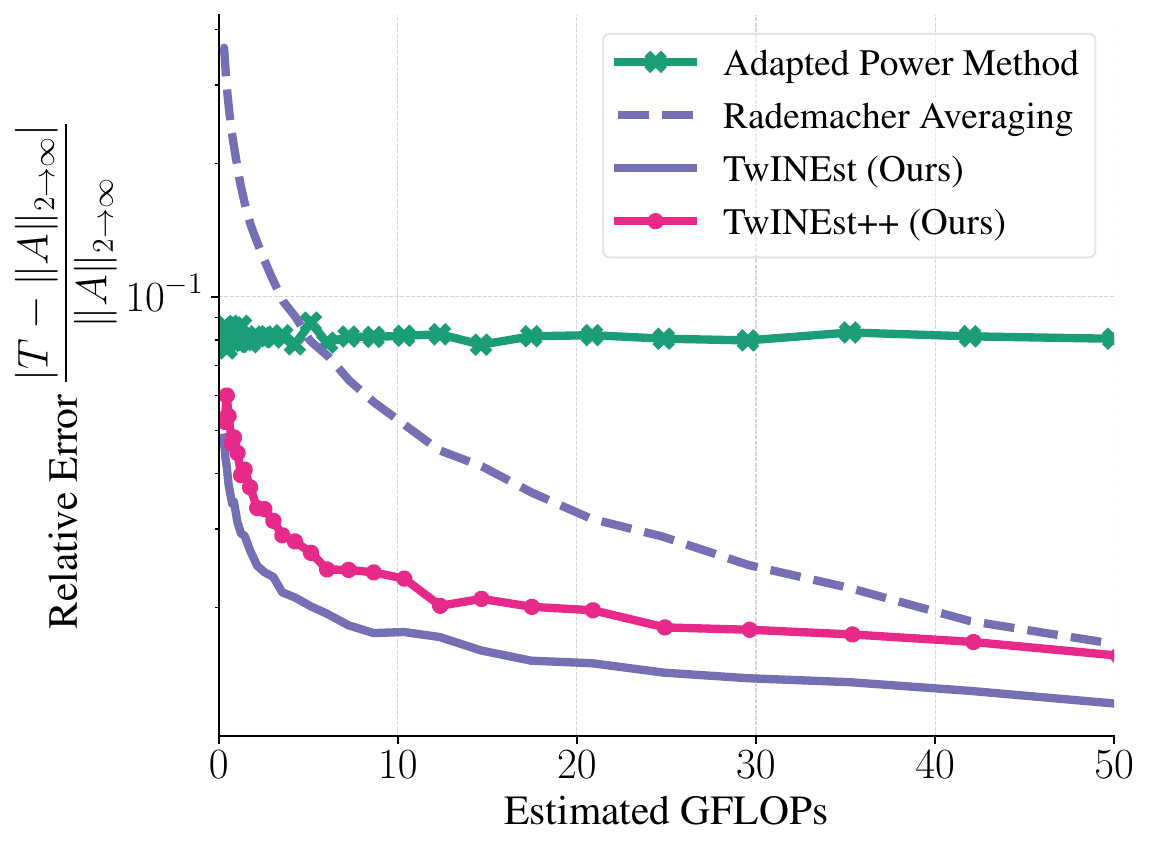}
    \subcaption{Synthetic data with $\Delta=10^{-2}$.}
   \end{subfigure}\hfill
   \begin{subfigure}{0.3\textwidth}
     \centering\includegraphics[width=\textwidth]{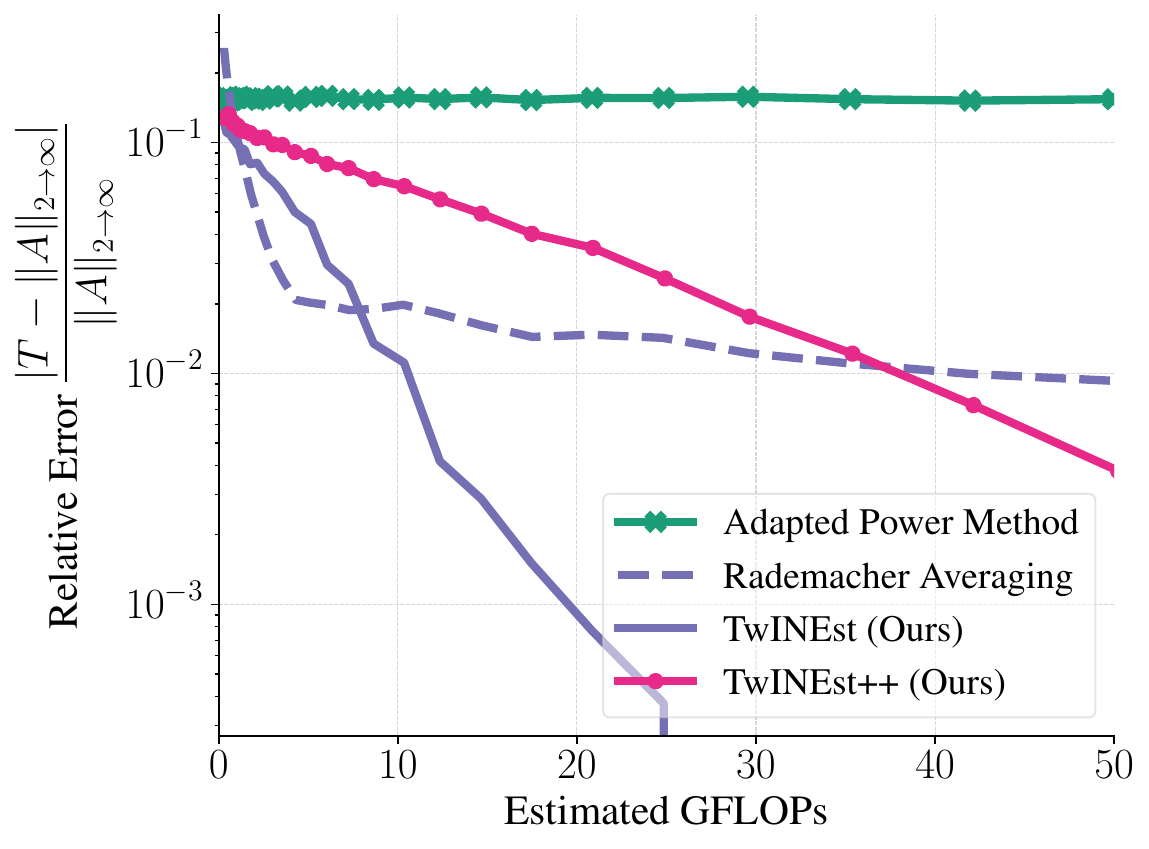}
    \subcaption{Synthetic data with $\Delta=10^{-1}$.}
   \end{subfigure}\hfill
   \begin{subfigure}{0.3\textwidth}
     \centering\includegraphics[width=\textwidth]{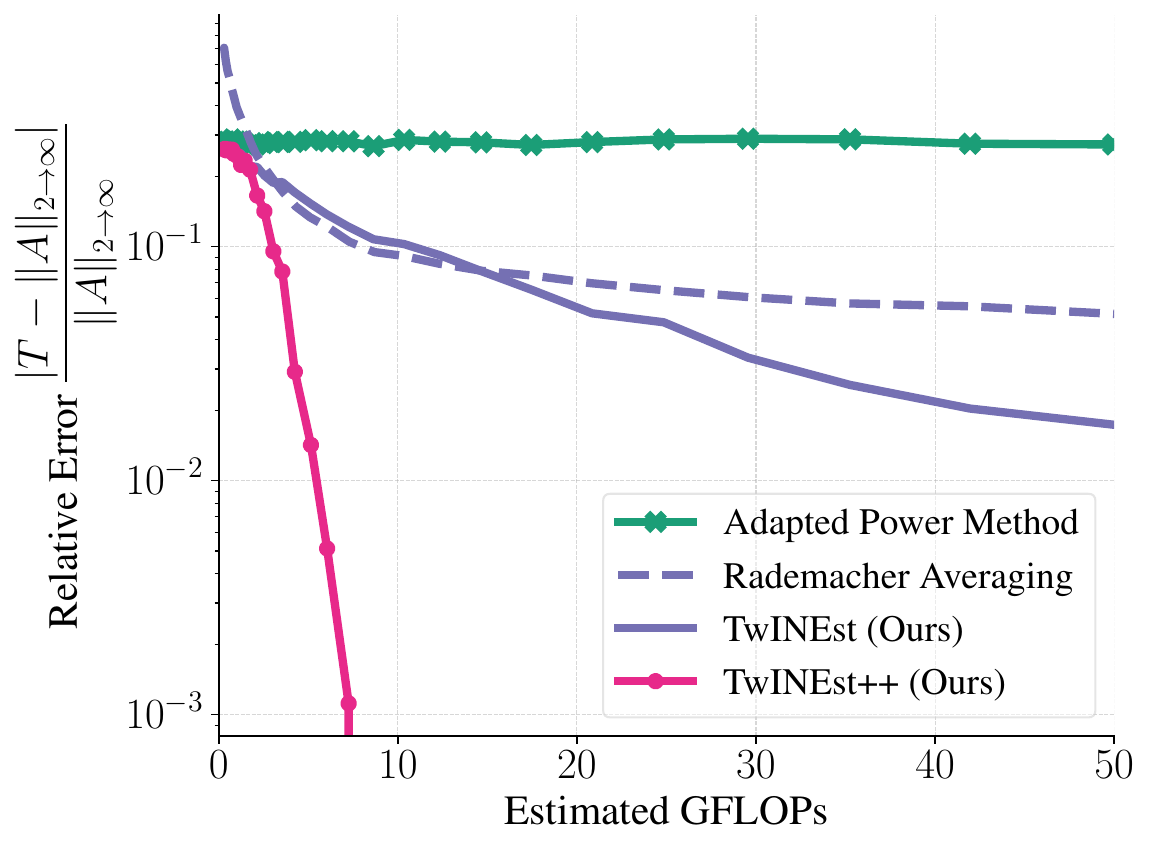}
    \subcaption{Jacobian matrix of WideResNet-16-10.}
   \end{subfigure}
   \caption{Comparison of methods for estimating the two-to-infinity matrix norm. The plot shows the relative error versus GFLOPs, averaged over 500 trials. For the Jacobian matrix, matrix-vector multiplications were computed using JVP and VJP via autograd, whereas for synthetic data, explicit matrix-vector multiplications were used.}
   \label{fig:flops}
\end{figure}

\subsection{Hyperparameters for Image Classification}
\label{appendix:hyperparams}

Each model is trained for 200 epochs using stochastic gradient descent (SGD) with Nesterov momentum of $0.9$ and weight decay of $5 \cdot 10^{-5}$. The initial learning rate is set to $0.1$, decayed by a factor of $0.1$ at epochs $60$, $120$, and $160$. We use a batch size of $128$ and apply the data augmentations listed in \Cref{tab:cifar_aug} and \Cref{tab:tiny_aug}. For both FGSM and PGD attacks, we use $\varepsilon = 2/255$.

\begin{table}[H]
\centering
\begin{tabular}{ll}
\toprule
\textbf{Transform} & \textbf{Parameters} \\
\midrule
RandomHorizontalFlip & --- \\
Pad & padding = 4,\quad padding\_mode = "symmetric" \\
RandomCrop & size = 32 \\
Normalize & mean = [0.5, 0.5, 0.5],\quad std = [0.5, 0.5, 0.5] \\
\bottomrule
\end{tabular}
\caption{Data augmentation used for CIFAR-100.}
\label{tab:cifar_aug}
\end{table}

\begin{table}[H]
\centering
\begin{tabular}{ll}
\toprule
\textbf{Transform} & \textbf{Parameters} \\
\midrule
RandomHorizontalFlip & --- \\
Pad & padding = 4,\quad padding\_mode = "symmetric" \\
RandomCrop & size = 64 \\
ColorJitter & brightness = 0.2,\quad contrast = 0.2,\quad \\
             & saturation = 0.2,\quad hue = 0.1 \\
Normalize & mean = [0.485, 0.456, 0.406],\quad std = [0.229, 0.224, 0.225] \\
\bottomrule
\end{tabular}
\caption{Data augmentation used for TinyImageNet.}
\label{tab:tiny_aug}
\end{table}

\subsection{Regularizers for Image Classification}
\label{appendix:reg_info}

For Frobenius-norm regularization, we use the algorithm introduced in \cite{hoffman2019robust}, relying on $1$ Jacobian–vector multiplication. For spectral and infinity norm regularization, we adopt the algorithm described in \cite{roth2020adversarial}, based on $3$ sequential Jacobian–vector multiplications. For two-to-infinity norm regularization, we employ the TwINEst algorithm, based on $5$ Jacobian–vector multiplications, $4$ of them executed in parallel.

\subsection{Regularizer Ablation}
\label{appendix:regularizer_ablation}

To demonstrate the practical applicability of our regularization, we show that updating the regularization term only once every $k$ iterations is sufficient to outperform other methods. \Cref{fig:twinest_ksteps} indicates that choosing $k = 50$ improves generalization ability by up to $1$ accuracy point while adding negligible wall-clock overhead compared with training without regularization.

\begin{figure}[H]
    \centering\includegraphics[width=0.45\textwidth]{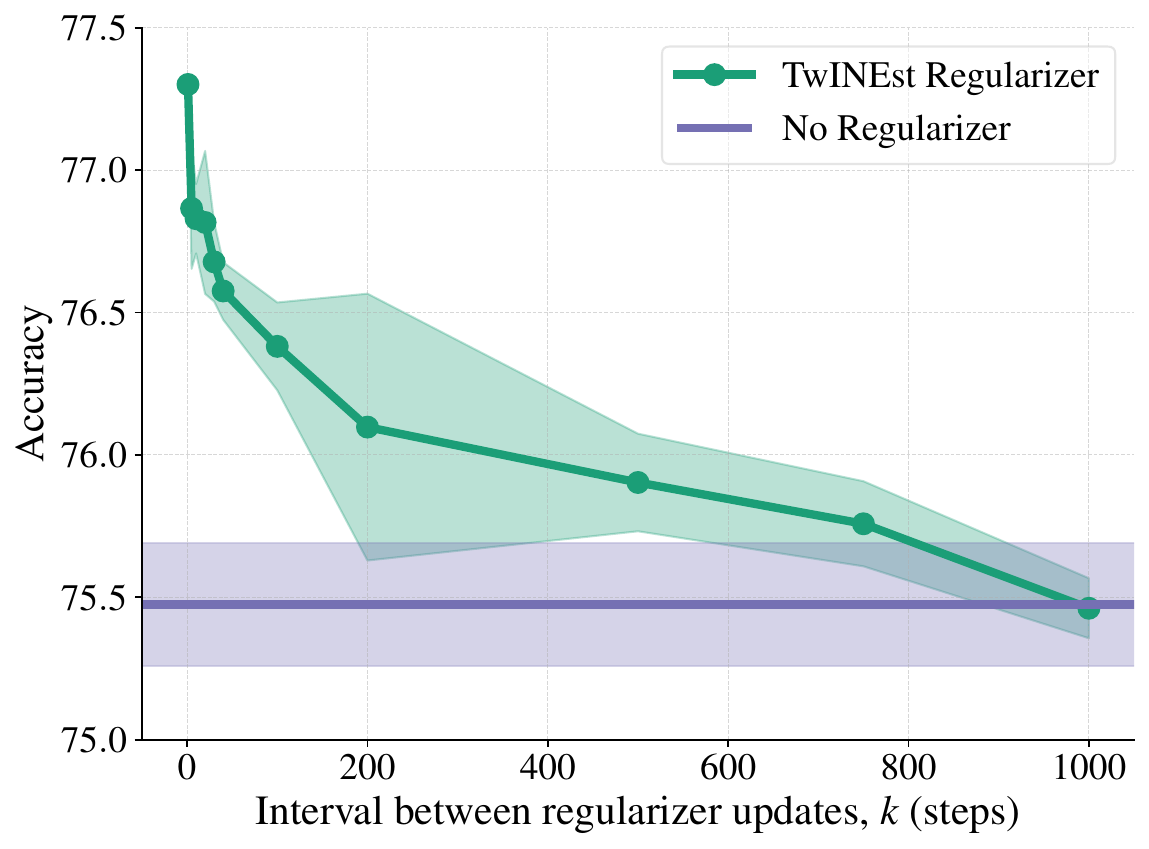}
    \caption{Comparison of different $k$ parameters in Jacobian two-to-infinity norm regularization method on the CIFAR-100 dataset using WideResNet-16-10. Accuracy metric is averaged over 3 trials.}
    \label{fig:twinest_ksteps}
\end{figure}

\subsection{Hyperparameters and Data Preprocessing for Recommender Systems}
\label{appendix:hyperparams_recsys}

\begin{table}[H]
\centering
\begin{tabular}{l|ccc|ccc|ccc}
\toprule
& \multicolumn{3}{c|}{\textbf{MovieLens-1M}}
& \multicolumn{3}{c|}{\textbf{Yelp-2018}}
  & \multicolumn{3}{c}{\textbf{CiteULike}} \\
\cmidrule{2-4} \cmidrule{5-7} \cmidrule{8-10}
\textbf{Hyperparameter} & W.D. & $2\to\infty$ & Max-norm& W.D. & $2\to\infty$ & Max-norm& W.D. & $2\to\infty$ & Max-norm\\
\midrule
Embedding dim & 64 & 64 & 64 & 64 & 128 & 128 & 128 & 128 & 128 \\
Learning rate & $10^{-3}$ & $10^{-2}$ & $10^{-4}$ & $10^{-3}$ & $10^{-3}$ & $10^{-3}$ & $10^{-4}$ & $10^{-4}$ & $10^{-3}$ \\
$\lambda$ & $0.0007$ & $0.03$ & $0.03$ & $0.004$ & $0.7$ & $0.0006$ & $0.09$ & $0.1$ & $0.02$ \\
$\gamma$ & $0.0001$ & $5\cdot 10^{-5}$ & $0.0002$ & $0.00005$ & $0.0005$ & $0.0009$ & $0.0002$ & $0.0001$ & $0.002$ \\
Negative num & 200 & 200 & 200 & 800 & 800 & 800 & 25 & 71 & 67 \\
Negative weight & 200 & 200 & 200 & 300 & 300 & 300 & 25 & 55 & 93 \\
\bottomrule
\end{tabular}
\caption{Hyperparameters for UltraGCN model.}
\label{tab:hyperparam_recsys}
\end{table}

For MovieLens-1M and Yelp2018 datasets we use train-test split from \cite{mao2021ultragcn}. For CiteULike dataset we filter users that have more than $2$ interactions and take $10\%$ random items from every user to test set.

\subsection{Evaluation Metrics and Additional Results for RecSys Application}
\label{appendix:add_metrics_recsys}

\begin{figure}[ht!]
    \begin{subfigure}{0.3\textwidth}
     \centering\includegraphics[width=\textwidth]{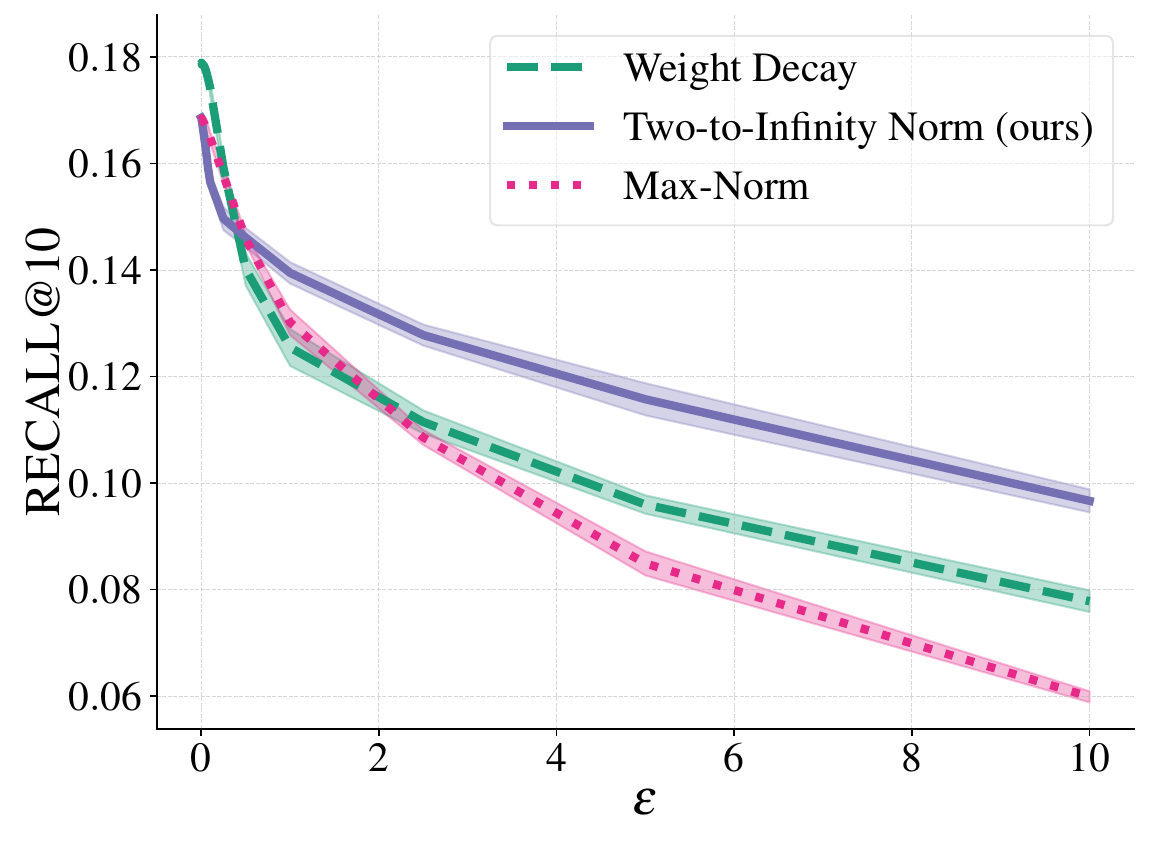}
   \end{subfigure}\hfill
   \begin{subfigure}{0.3\textwidth}
     \centering\includegraphics[width=\textwidth]{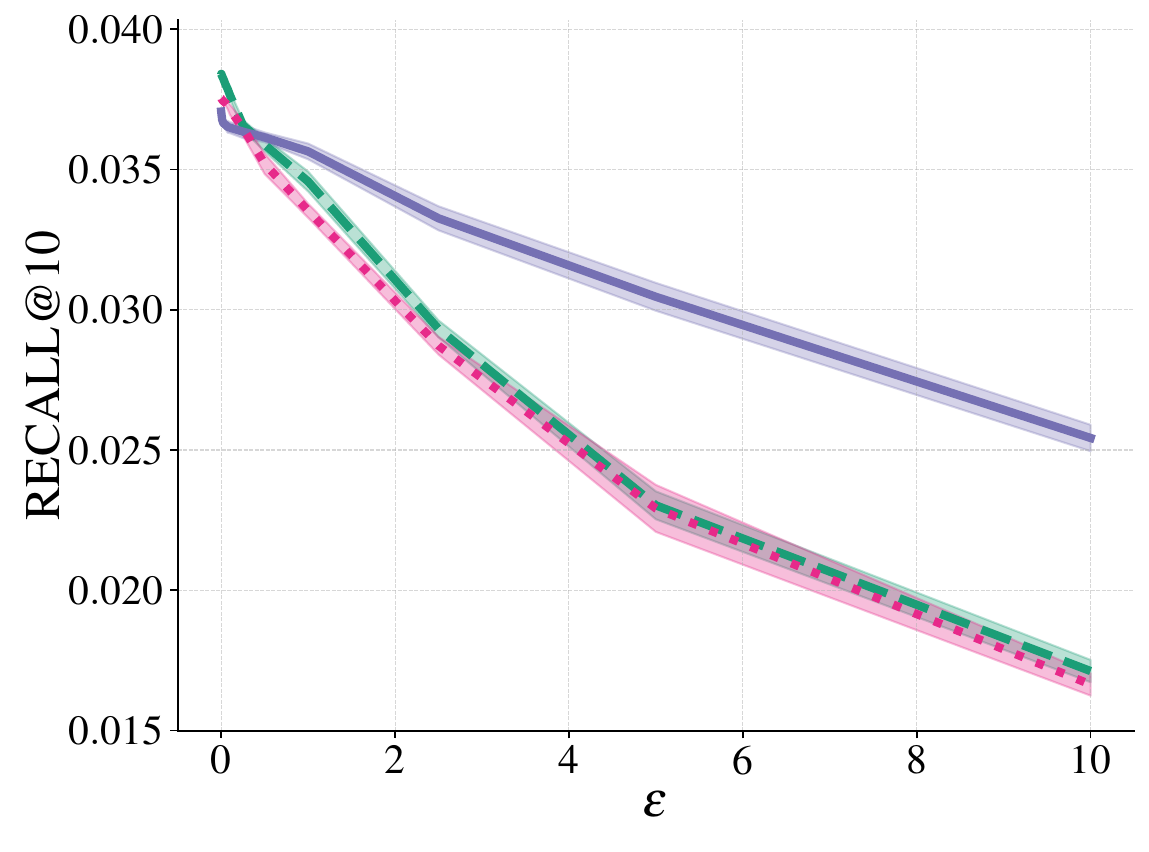}
   \end{subfigure}\hfill
   \begin{subfigure}{0.3\textwidth}
     \centering\includegraphics[width=\textwidth]{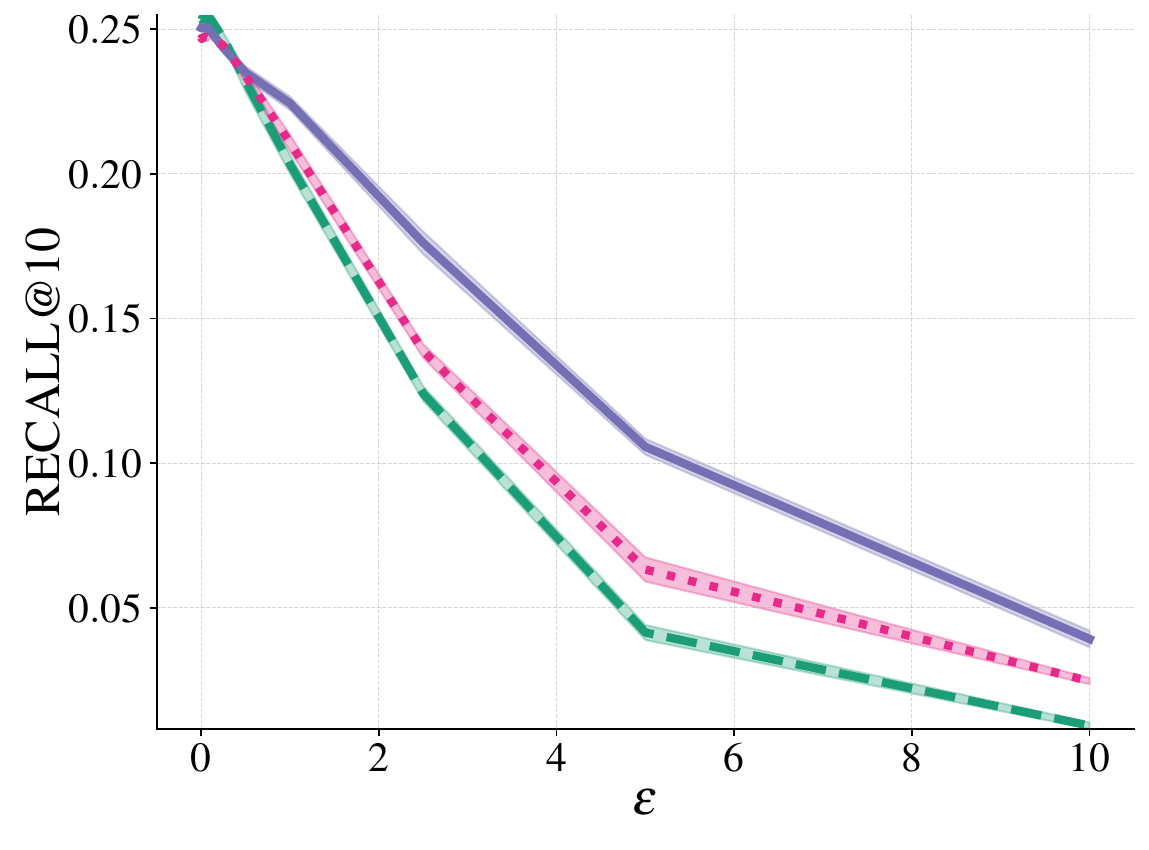}
   \end{subfigure}\vfill
   \begin{subfigure}{0.3\textwidth}
     \centering\includegraphics[width=\textwidth]{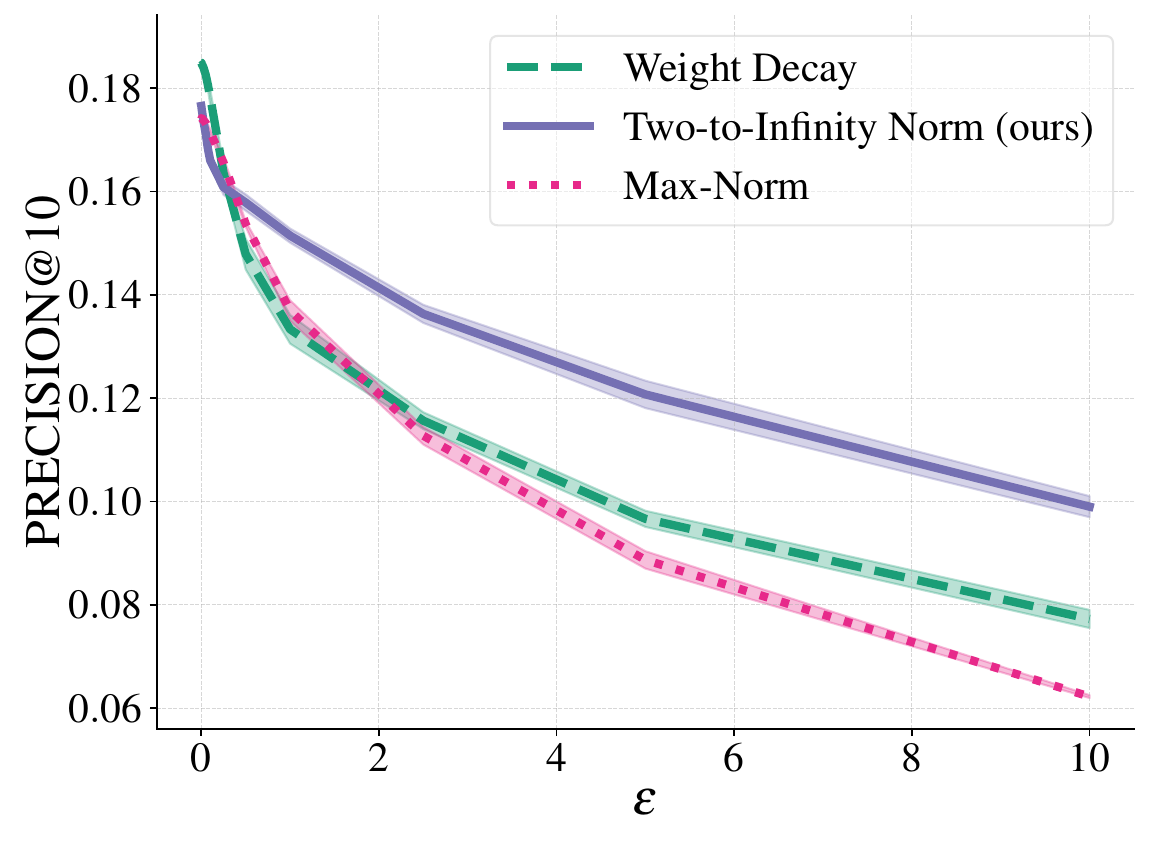}
   \end{subfigure}\hfill
   \begin{subfigure}{0.3\textwidth}
     \centering\includegraphics[width=\textwidth]{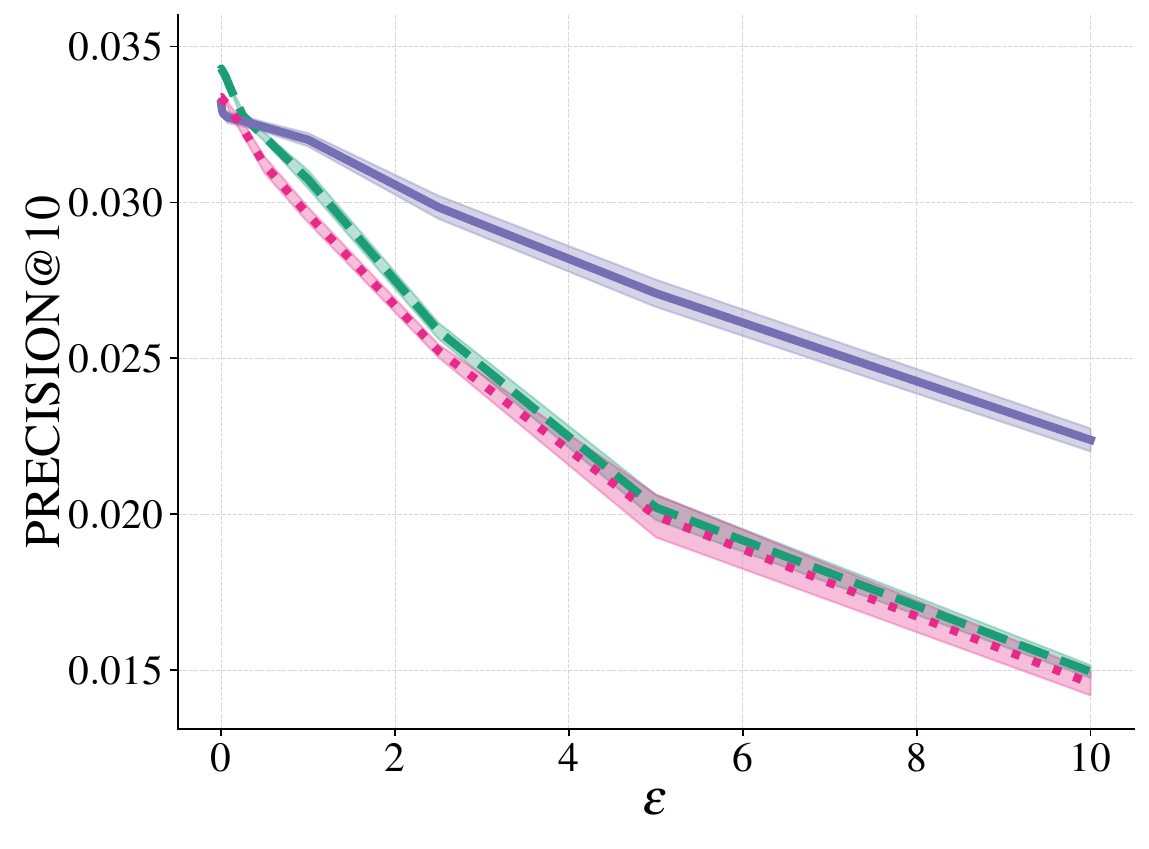}
   \end{subfigure}\hfill
   \begin{subfigure}{0.3\textwidth}
     \centering\includegraphics[width=\textwidth]{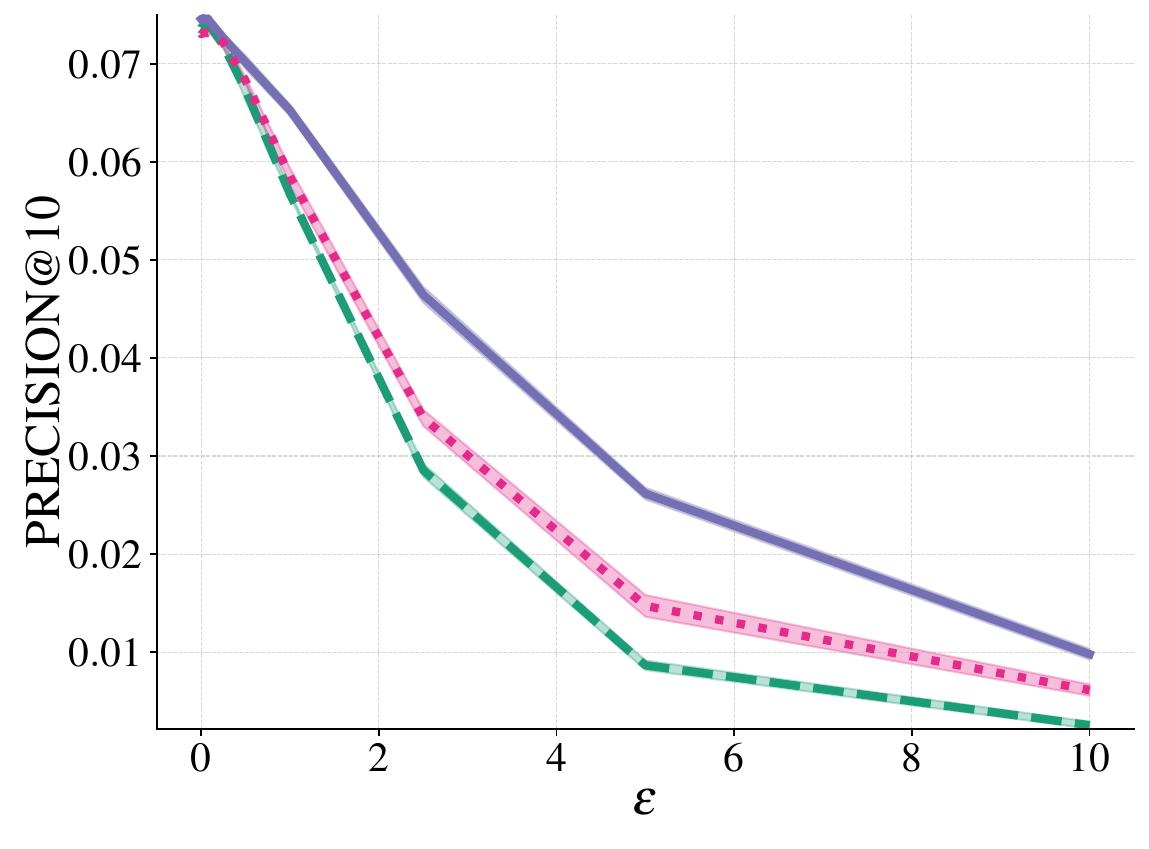}
   \end{subfigure}\vfill
    \begin{subfigure}{0.3\textwidth}
     \centering\includegraphics[width=\textwidth]{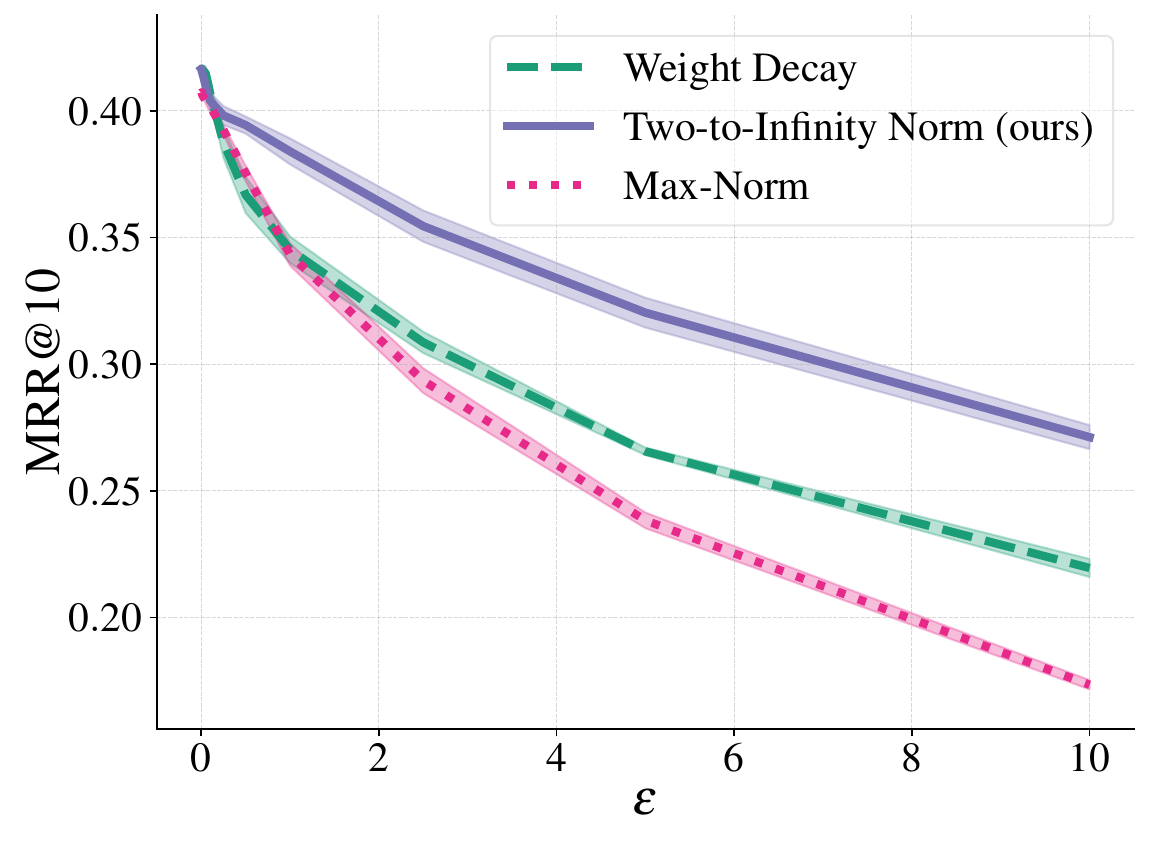}
   \end{subfigure}\hfill
   \begin{subfigure}{0.3\textwidth}
     \centering\includegraphics[width=\textwidth]{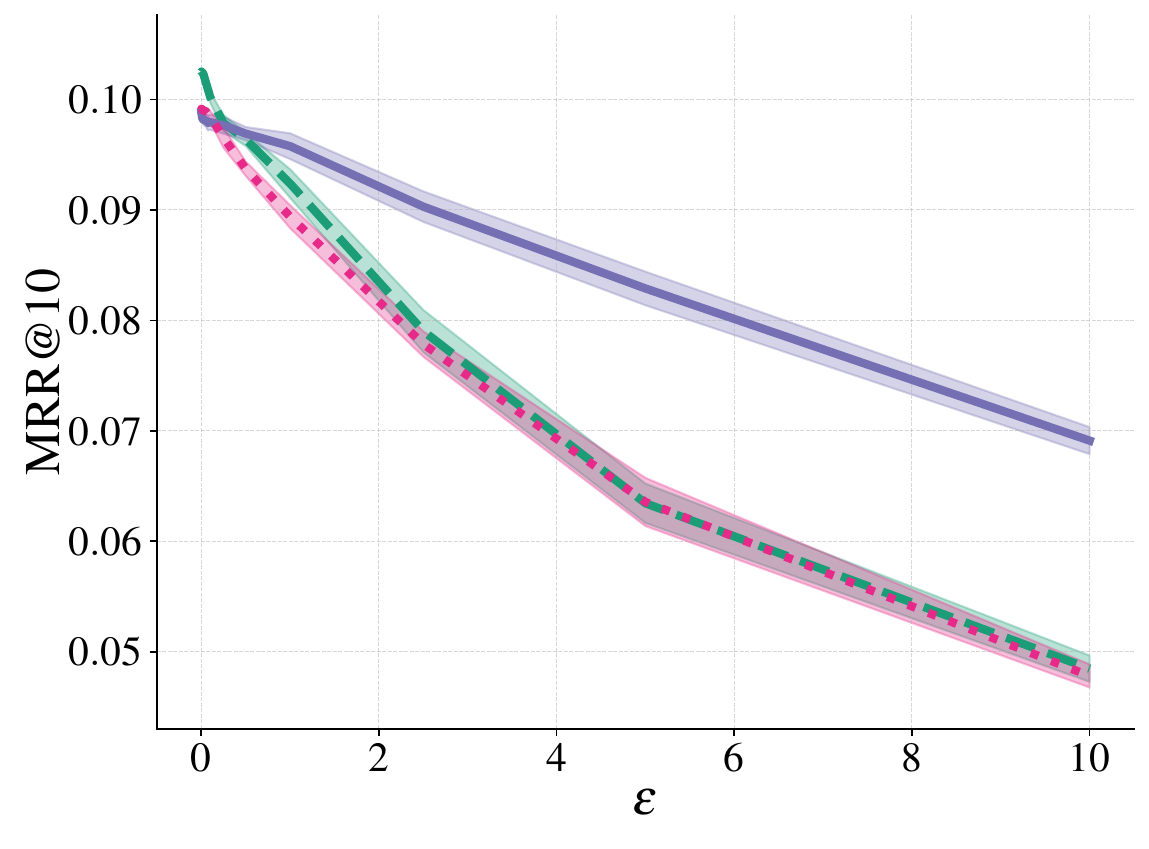}
   \end{subfigure}\hfill
   \begin{subfigure}{0.3\textwidth}
     \centering\includegraphics[width=\textwidth]{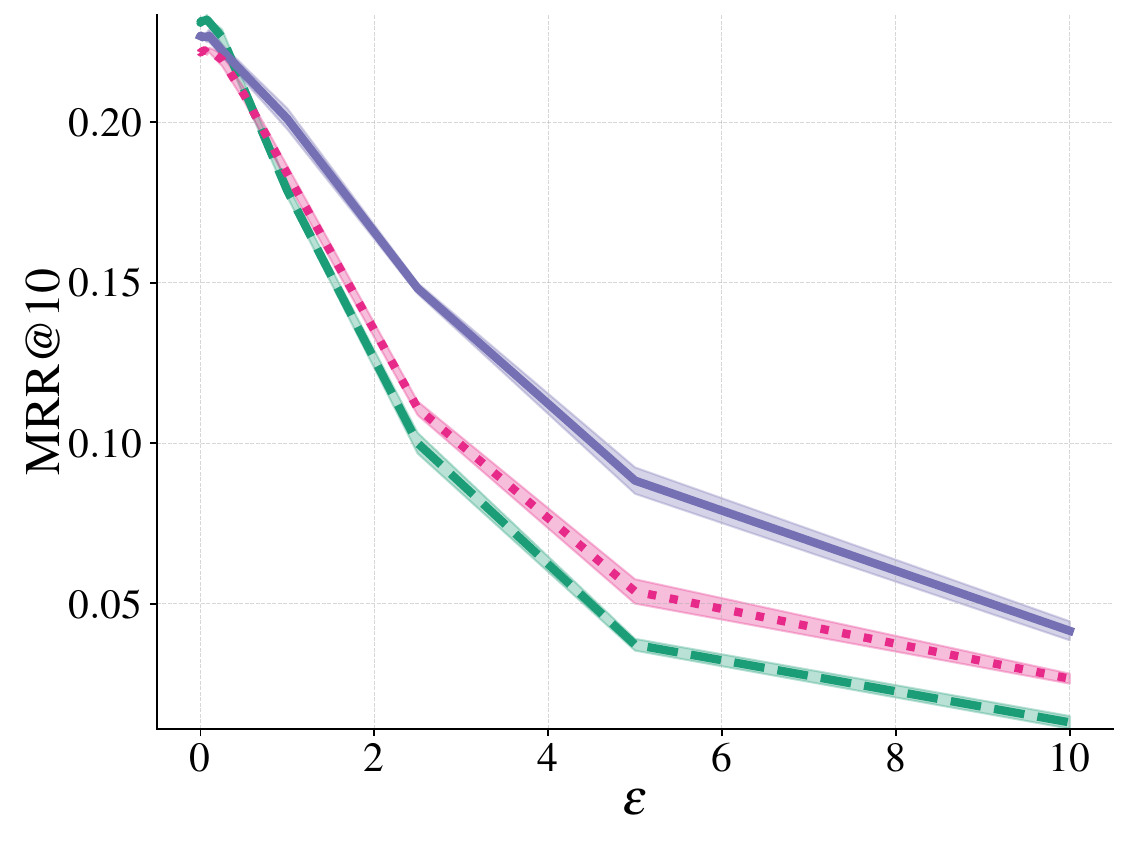}
   \end{subfigure}\vfill
    \begin{subfigure}{0.3\textwidth}
     \centering\includegraphics[width=\textwidth]{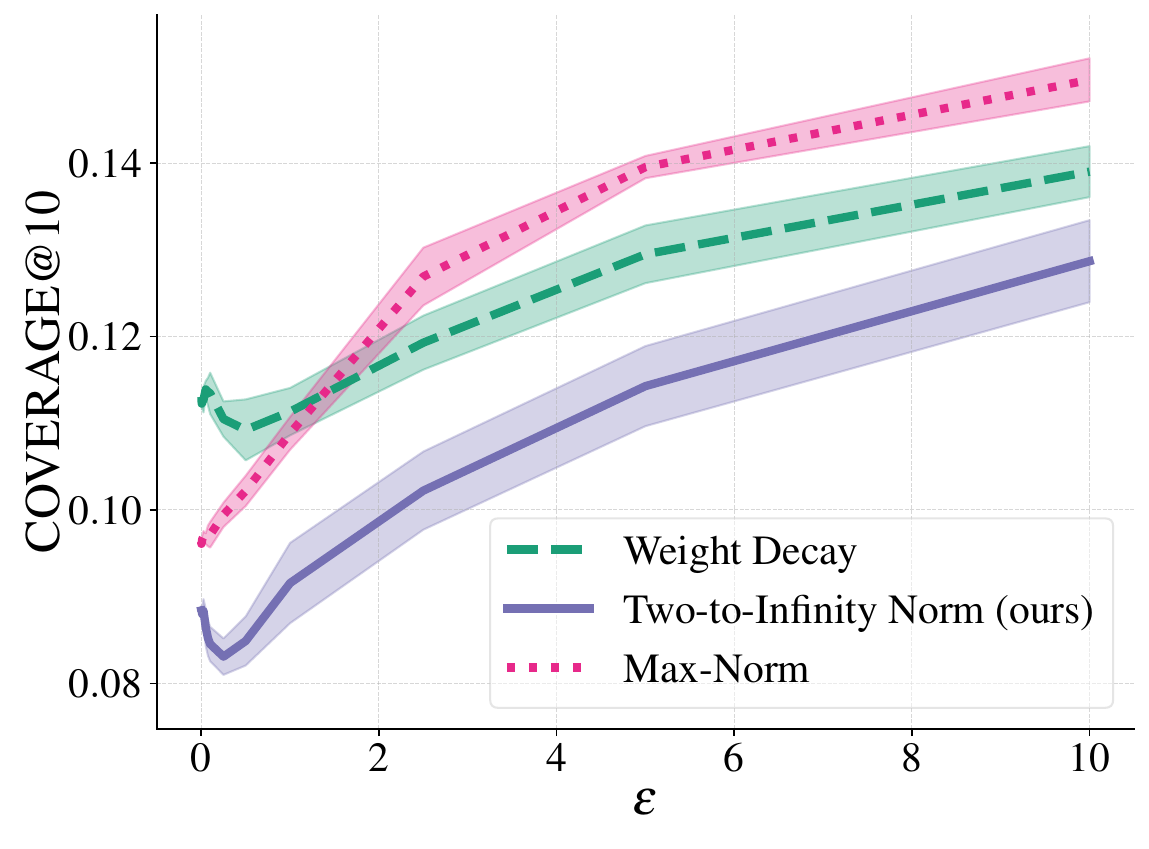}
    \subcaption{MovieLens-1M dataset.}
   \end{subfigure}\hfill
   \begin{subfigure}{0.3\textwidth}
     \centering\includegraphics[width=\textwidth]{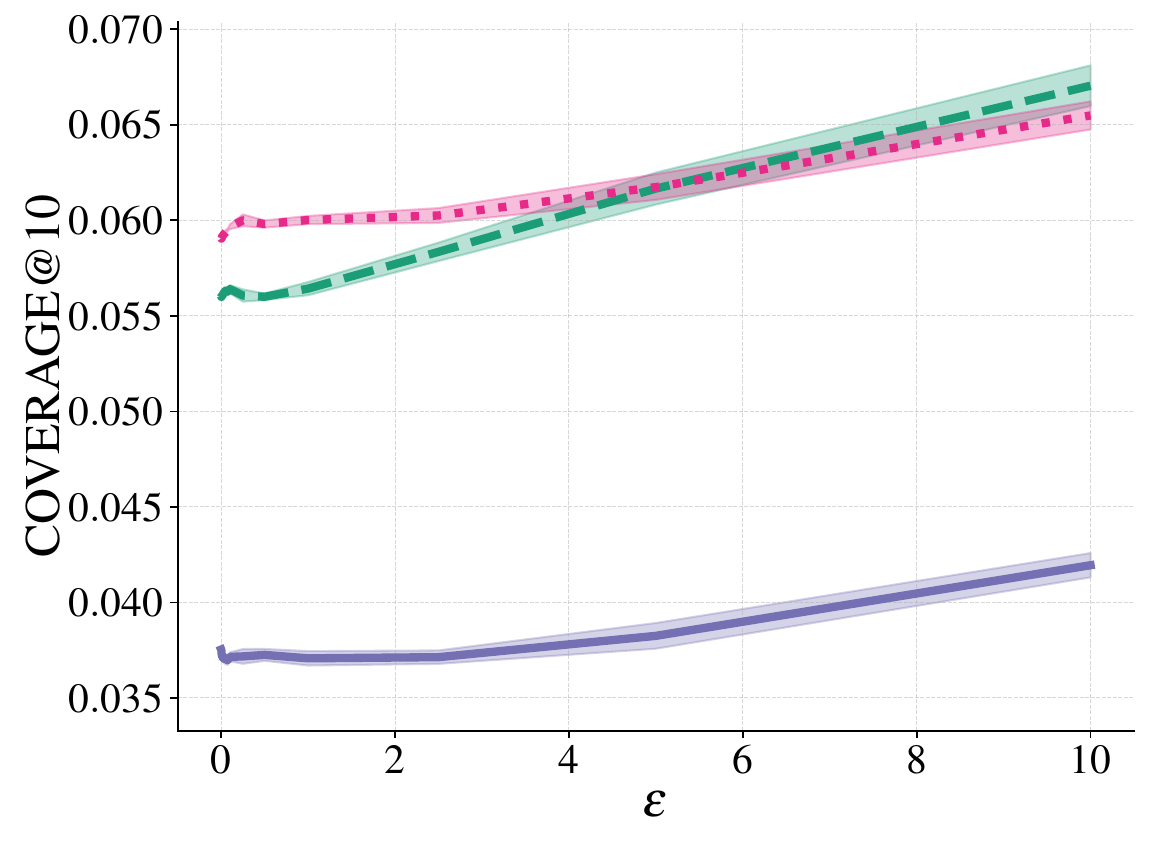}
    \subcaption{Yelp2018 dataset.}
   \end{subfigure}\hfill
   \begin{subfigure}{0.3\textwidth}
     \centering\includegraphics[width=\textwidth]{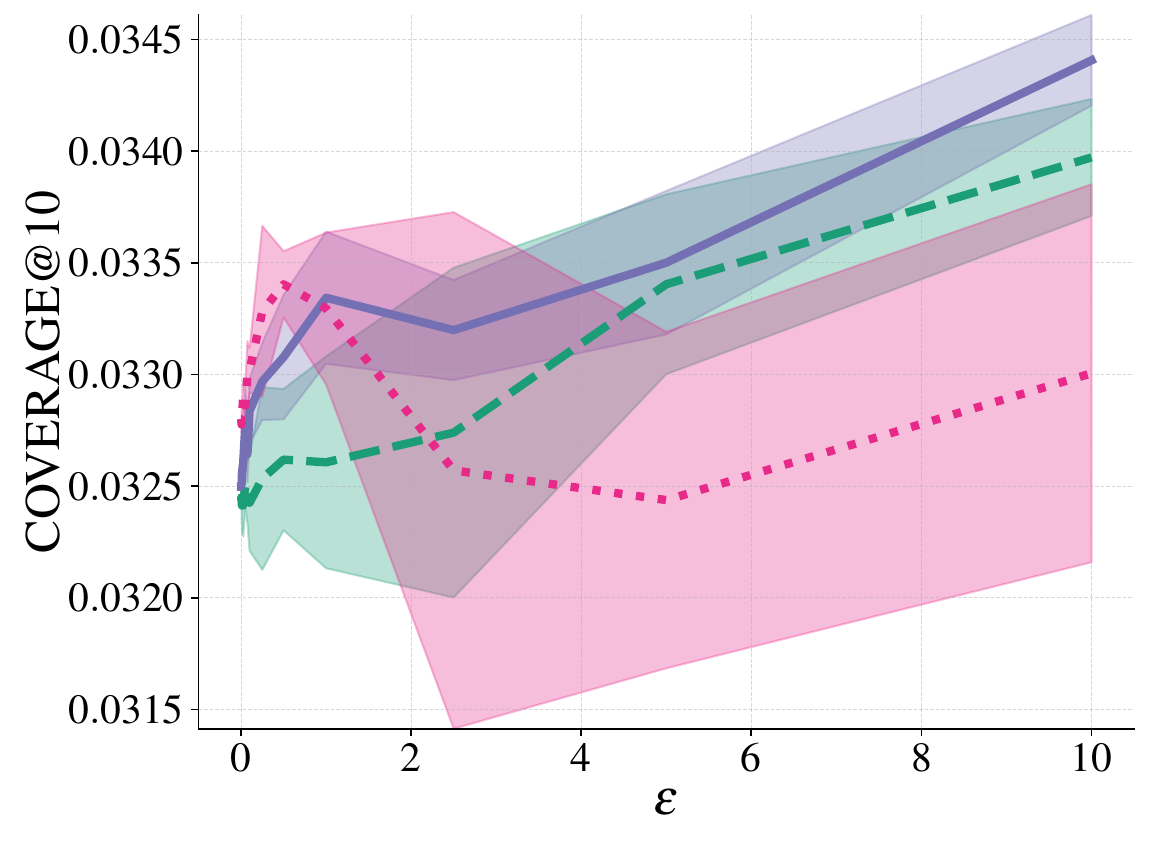}
    \subcaption{CiteULike dataset.}
   \end{subfigure}
   \caption{Comparison of different regularization methods for adversarial robustness of the UltraGCN model across various metrics. Metrics are averaged over 5 trials.}
    \label{fig:recsys_additional_metrics}
\end{figure}

\begin{definition}[Normalized Discounted Cumulative Gain (NDCG)]
Consider the collaborative filtering setting described in \Cref{exp:recsys}. For user $u$, let $i_1, \dots, i_k$ denote the items with the highest predicted scores in $\hat{R}_u$, and let
$$\mathcal{I} = \{i \mid R_{ui} = 1\}$$ 
denote the set of items that interacted with user $u$.

\noindent\textbf{Discounted Cumulative Gain.}
The \emph{Discounted Cumulative Gain} at rank cutoff $k$ is
\begin{equation*}
\mathrm{DCG}@k
=
\sum_{j=1}^{k} \frac{\mathbb{I} (i_j \in \mathcal{I})}{\log_2(j+1)} .
\end{equation*}

\noindent\textbf{Ideal DCG.}
The \emph{Ideal DCG} at rank cutoff $k$ is the maximum DCG attainable for this user:
\begin{equation*}
\mathrm{IDCG}@k
=
\sum_{j=1}^{\min(k, |\mathcal{I}|)} \frac{1}{\log_2(j+1)} .
\end{equation*}

\noindent\textbf{Normalized DCG.}
The \emph{Normalized Discounted Cumulative Gain} at cutoff $k$ is
\begin{equation*}
\mathrm{NDCG}@k
=
\frac{\mathrm{DCG}@k}{\mathrm{IDCG}@k}.
\end{equation*}
\end{definition}

\begin{definition}[Mean Reciprocal Rank (MRR)]
For user $u$, let $i_1, \dots, i_k$ denote the items with the highest predicted scores in $\hat{R}_u$, and let
$$\mathcal{I} = \{i \mid R_{ui} = 1\}$$ 
denote the set of items that interacted with user $u$.

\noindent\textbf{Reciprocal Rank.}  
Let
\begin{equation*}
j_u^\ast \;=\; \min\{\,j \in \{1,\dots,k\} \mid i_j \in \mathcal{I}_u\},
\end{equation*}
i.e. the rank position of the first relevant item within the top-$k$ list (set $j_u^\ast = \infty$ if no relevant item appears).  
The \emph{Reciprocal Rank} for user $u$ at cutoff $k$ is then
\begin{equation*}
\mathrm{RR}_u@k
=
\begin{cases}
\dfrac{1}{j_u^\ast}, & j_u^\ast < \infty\\
0, & j_u^\ast = \infty.
\end{cases}
\end{equation*}

\noindent\textbf{Mean Reciprocal Rank.}  
Given the evaluation user set $\mathcal{U}$, the \emph{Mean Reciprocal Rank} is the average of the individual reciprocal ranks:
\begin{equation*}
\mathrm{MRR}@k
=
\frac{1}{|\mathcal{U}|}
\sum_{u \in \mathcal{U}}
\mathrm{RR}_u@k.
\end{equation*}
\end{definition}

\begin{definition}[Coverage]
Let $d_i$ denote the number of all items,
$\mathcal{U}$ be the evaluation user set and
$i^{(u)}_1,\dots,i^{(u)}_k$ the $k$ highest-scoring items
recommended to user $u \in \mathcal{U}$.
Define the set of all recommended items
\begin{equation*}
\mathcal{S}_k
=
\bigcup_{u\in\mathcal{U}}
\bigl\{i^{(u)}_1,\dots,i^{(u)}_k\bigr\}.
\end{equation*}

\noindent\textbf{Coverage.}
The \emph{coverage} at cutoff~$k$ is the fraction of the catalogue
that the recommender exposes across all top-$k$ lists:
\begin{equation*}
\mathrm{Coverage}@k
=
\frac{\lvert\mathcal{S}_k\rvert}{d_i}.
\end{equation*}
\end{definition}

\end{document}